\documentclass[twoside]{article}
\usepackage[accepted]{publicpaper}
\usepackage[utf8]{inputenc}
\usepackage{amsthm, amssymb, amsmath}
\usepackage{mathtools}
\usepackage{graphicx}
\usepackage[ruled]{algorithm2e}
\usepackage[breaklinks, colorlinks=true,linkcolor=black, citecolor=blue]{hyperref}
\usepackage{dsfont}
\usepackage{subcaption}
\usepackage{chngcntr}
\usepackage{apptools}
\usepackage{mathtools}
\usepackage[round]{natbib}
\usepackage[final]{fixme}
\usepackage{tikz}
\usetikzlibrary{positioning}
\fxsetup{inline,nomargin,theme=color}
\AtAppendix{\counterwithin{lemma}{section}}

\newcommand{\uml}{\textsc{Uniform Metric Labeling}}

\newcommand{\bg}{$(\beta, \gamma)$}

\DeclareMathOperator*{\argmin}{arg\,min}

\begin{document}
\newtheorem{lemma}{Lemma}
\newtheorem{corollary}{Corollary}
\newtheorem{theorem}{Theorem}
\newtheorem*{itheorem}{Informal Theorem}
\newtheorem{proposition}{Proposition}
\newtheorem{property}{Property}
\theoremstyle{definition}
\newtheorem{definition}{Definition}

\twocolumn[
\aistatstitle{Block Stability for MAP Inference}
\aistatsauthor{Hunter Lang \And David Sontag \And Aravindan Vijayaraghavan}
\aistatsaddress{MIT \And MIT \And Northwestern University} ]

\begin{abstract}
  To understand the empirical success of approximate MAP inference,
  recent work \citep{LanSonVij18} has shown that some popular
  approximation algorithms perform very well when the input instance
  is \emph{stable}. The simplest stability condition assumes that the
  MAP solution does not change at all when some of the pairwise
  potentials are (adversarially) perturbed. Unfortunately, this strong
  condition does not seem to be satisfied in practice. In this paper,
  we introduce a significantly more relaxed condition that only
  requires blocks (portions) of an input instance to be stable. Under
  this block stability condition, we prove that the pairwise LP
  relaxation is \emph{persistent} on the stable blocks. We complement
  our theoretical results with an empirical evaluation of real-world
  MAP inference instances from computer vision. We design an algorithm
  to find stable blocks, and find that these \emph{real} instances
  have large stable regions. Our work gives a theoretical explanation
  for the widespread empirical phenomenon of persistency for this LP
  relaxation.
\end{abstract}

\section{INTRODUCTION}
\label{sec:intro}
As researchers and practitioners begin to apply machine learning
algorithms to areas of society where human lives are at stake---such
as bail decisions, autonomous vehicles, and healthcare---it becomes
increasingly important to understand the empirical performance of
these algorithms from a theoretical standpoint. Because many machine
learning problems are NP-hard, the approaches deployed in practice are
often heuristics or approximation algorithms. These sometimes come
with performance guarantees, but the algorithms typically do
\emph{much better} in practice than their theoretical guarantees
suggest. Heuristics are often chosen solely on the basis of their past
empirical performance, and our theoretical understanding of the
reasons for such performance is limited. To design better algorithms
and to better understand the strengths of our existing approaches, we
must bridge this gap between theory and practice.

To this end, many researchers have looked \emph{beyond worst-case
  analysis}, developing approaches like smoothed analysis,
average-case analysis, and stability. Broadly, these approaches all
attempt to show that the worst-case behavior of an algorithm does not
occur too often in the real world. Some methods are able to show that
worst-case instances are ``brittle,'' whereas others show that
real-world instances have special structure that makes the problem
significantly easier. In this work, we focus on stability, which takes
the latter approach. Informally, an instance of an optimization
problem is said to be stable if the (optimal) solution does not change
when the instance is perturbed. This captures the intuition that
solutions should ``stand out'' from other feasible points on
real-world problem instances.

We focus on the MAP inference problem in Markov Random Fields. MAP
inference is often used to solve \emph{structured prediction} problems
like stereo vision. The goal of stereo vision is to go from two images---one
taken from slightly to the right of the other, like the images seen by
your eyes---to an \emph{assignment} of depths to pixels, which
indicates how far each pixel is from the camera. Markov Random Fields
give an elegant method for finding the best assignment of states
(depths) to variables (pixels), taking into account the structure of
the output space. Figure \ref{fig:gapfig} illustrates the need for
a better theoretical understanding of MAP inference algorithms. An
exact solution to the MAP problem for a real-world stereo vision
instance appears in Figure \ref{fig:gapfig_opt}. Figure
\ref{fig:gapfig_bad} shows an assignment that, according to the
current theory, might be returned by the best approximation
algorithms. These two assignments agree on less than 1\% of their
labels. Finally, Figure \ref{fig:gapfig_appx} shows an assignment
\emph{actually} returned by an approximation algorithm---this
assignment has over 99\% of labels in common with the exact one. This
surprising behavior is not limited to stereo vision. Many structured
prediction problems have approximate MAP algorithms that perform
extremely well in practice despite the exact MAP problems being
NP-hard \citep{Koo, globalMAP, KomoSurvey, swobodapartial}.

The huge difference between Figures \ref{fig:gapfig_bad} and
\ref{fig:gapfig_appx} indicates that real-world instances must have
some structure that makes the MAP problem easy.  Indeed, these
instances seem to have some stability to multiplicative
perturbations. Figure \ref{fig:stabfig} shows MAP solutions to a
stereo vision instance and a small perturbation of that
instance.\footnote{The instances in Figures \ref{fig:gapfig} and
  \ref{fig:stabfig} have the same input images, but Figure
  \ref{fig:stabfig} uses higher resolution.} These solutions share
many common labels, and many portions are exactly the same.

\begin{figure}[t!]
  \centering
  \begin{subfigure}[t]{.32\linewidth}
    \centering
    \includegraphics[width=\linewidth]{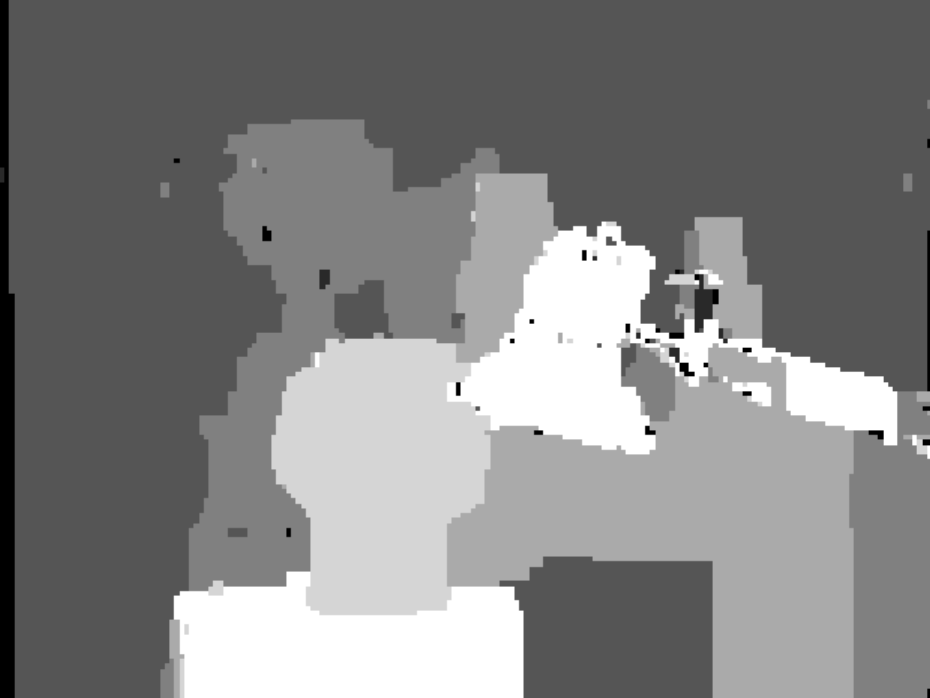}
    \subcaption{Exact soln.}
    \label{fig:gapfig_opt}
  \end{subfigure}%
  \begin{subfigure}[t]{.32\linewidth}
    \centering
    \includegraphics[width=\linewidth]{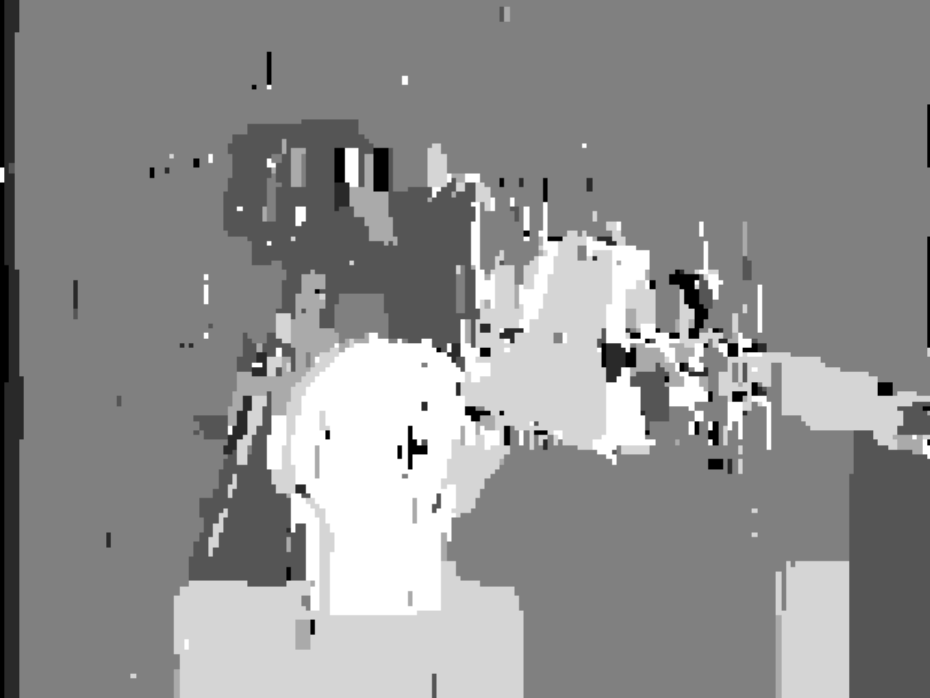}
    \subcaption{Best theory}
    \label{fig:gapfig_bad}    
  \end{subfigure}%
  \begin{subfigure}[t]{.32\linewidth}
    \centering
    \includegraphics[width=\linewidth]{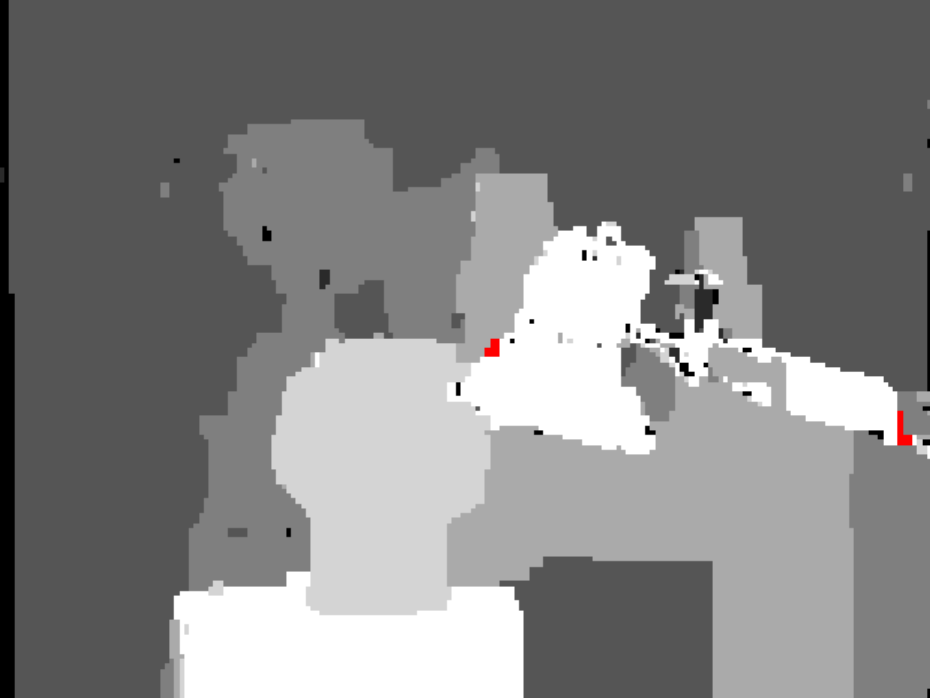}
    \subcaption{Actual appx. soln.}
    \label{fig:gapfig_appx}
  \end{subfigure}
  \caption{Example of an exact solution (left) to a stereo vision
    MAP problem compared to a 2-approximation (the best known theoretical performance bound, center), and a real approximate solution returned by the LP relaxation (right). Fractional portions of the LP solution are shown in red.}
  \label{fig:gapfig}
\end{figure}

\begin{figure}[t!]
  \centering
  \begin{subfigure}[t]{.5\linewidth}
    \centering
    \includegraphics[width=\linewidth]{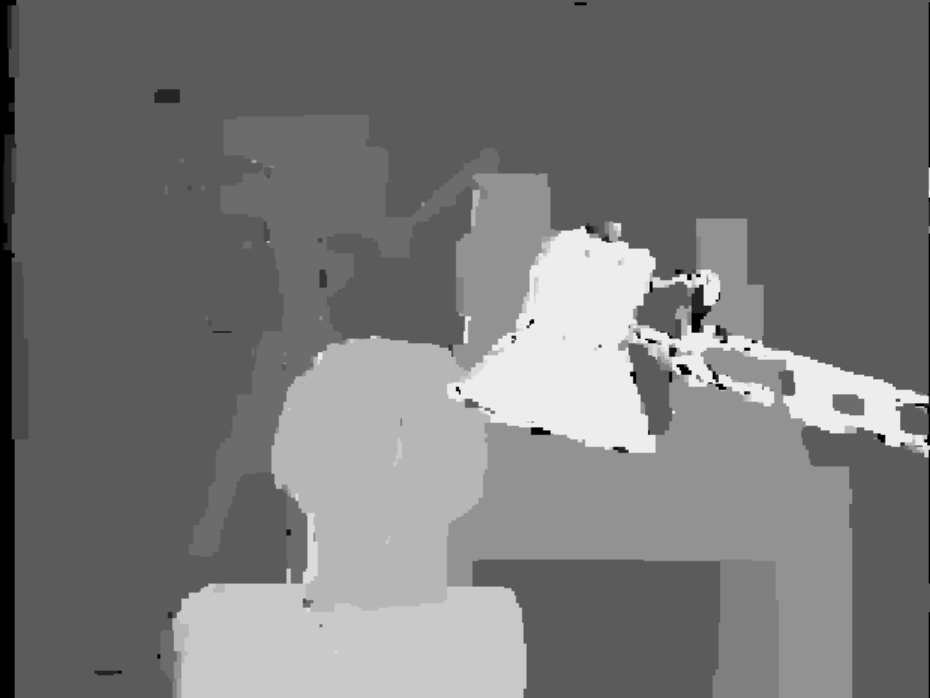}
    \subcaption{Original solution}
  \end{subfigure}%
  \begin{subfigure}[t]{.5\linewidth}
    \centering
    \includegraphics[width=\linewidth]{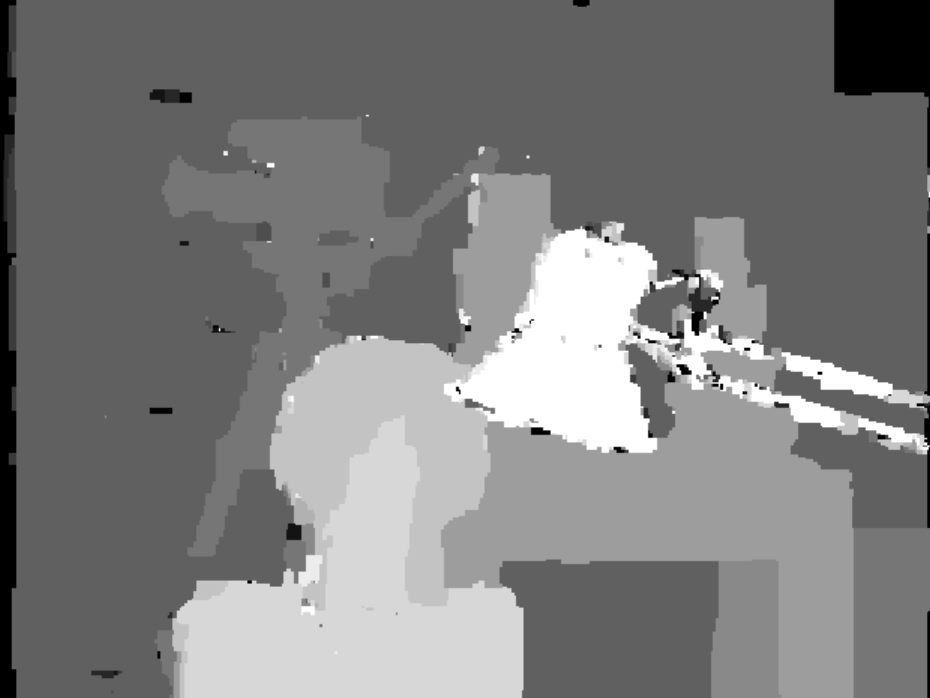}
    \subcaption{Solution to perturbation}
  \end{subfigure}%
  \caption{Solutions to an original (left) and multiplicatively perturbed (right) stereo vision instance. The two solutions agree on over 95\% of the vertices.}
  \label{fig:stabfig}
\end{figure}

Put simply, in the remainder of this work we attempt to use the
structure depicted in Figure \ref{fig:stabfig} to explain why Figure
\ref{fig:gapfig_appx} is so similar to Figure \ref{fig:gapfig_opt}.

The approximation algorithm used to produce Figure
\ref{fig:gapfig_appx} is called the \emph{pairwise LP relaxation}
\citep{wainwrightjordan, firstearthmover}. This algorithm formulates MAP inference as
an integer linear program (ILP) with variables $x$ that are
constrained to be in $\{0,1\}$. It then relaxes that ILP to a linear
program (LP) with constraints $x \in [0,1]$, which can be solved
efficiently. Unfortunately, the LP solution may not be a valid MAP
solution---it may have fractional values $x \in (0,1)$---so it might
need to be \emph{rounded} to a MAP solution. However, in practice, the
LP solution frequently takes values in $\{0,1\}$, and these values
``match'' with the exact MAP solution, so very little rounding is
needed. For example, the LP solution shown in Figure
\ref{fig:gapfig_appx} takes binary values that agree with the exact
solution on more than 99\% of the instance. This property is known as
\emph{persistency}.

Much previous work has gone into understanding the persistency of the
LP relaxation, typically stemming from a desire to give partial
optimality guarantees for LP solutions and to use persistent solutions as a building block for finding exact MAP solutions. These results \emph{use} the fact that
the pairwise LP is often persistent on large portions of these
instances to design fast algorithms for verifying partial
optimality and for exact MAP inference \citep{kovtun2003partial, globalMAP, swobodapartial, haller2018exact, shekhovtsov2018maximum}. Contrastingly, our work aims to understand \emph{why} the
LP is persistent so frequently on real-world instances.

\citet{LanSonVij18} first explored the stability framework of
\citet{MMWC4Stab} in the context of MAP inference. They showed that
under a strong stability condition, the pairwise LP relaxation
provably returns an exact MAP solution. Unfortunately, this condition
(that the solution does not change \emph{at all} under perturbations)
is rarely, if ever, satisfied in practice. On the other hand, Figure
\ref{fig:stabfig} demonstrates that the original and perturbed
solutions do have many labels in common, so there could be some
stability present at the ``sub-instance'' level.

In this work, we give an extended stability framework, generalizing
the work of \citet{LanSonVij18} to the setting where only some parts
of an instance have stable structure. This naturally connects to work
on dual decomposition for MAP inference. We establish a theoretical
connection between dual decomposition and stability, which allows us
to use stability even when it is only present on parts of an instance,
and allows us to combine stability with other reasons for
persistency. In particular, we define a new notion called \emph{block
  stability}, for which we show the following:
\begin{itemize}
  \setlength\itemsep{-0.3em}
\item We prove that approximate solutions returned by the pairwise LP
  relaxation agree with the exact solution on all the stable blocks of
  an instance.
\item We design an algorithm to find stable blocks on real-world
  instances.
\item We run this algorithm on several examples from low-level
  computer vision, including stereo vision, where we find that these
  instances contain large stable blocks.
\item We demonstrate that the block framework can be used to
  incorporate stability with other structural reasons for persistency of the LP
  relaxation.
\end{itemize}
Taken together, these results suggest that block stability is a
plausible explanation for the empirical success of LP-based algorithms
for MAP inference.

\section{BACKGROUND}
\subsection{MAP Inference and Metric Labeling}
A Markov Random Field consists of a graph $G = (V,E)$, a
discrete set of labels $L = \{1, \ldots, k\}$, and \emph{potential
  functions} $\theta$ that capture the cost of assignments $f: V \to
L$. The MAP inference task in a Markov Random Field is to find the
assignment (or \emph{labeling}) $f: V \to L$ with the lowest cost:
\begin{equation}
  \label{eq:mapobj}
  \min_{f: V \to L} \sum_{u \in V}\theta_u(f(u)) + \sum_{uv \in E}\theta_{uv}(f(u), f(v)).
\end{equation}
Here we have decomposed the set of potential functions functions
$\theta$ into $\theta_u$ and $\theta_{uv}$, which correspond to nodes
and edges in the graph $G$, respectively. A Markov Random Field that
can be decomposed in this manner is known as a \emph{pairwise} MRF;
we focus exclusively on pairwise MRFs. In equation \eqref{eq:mapobj},
the \emph{single-node} potential functions $\theta_u(i)$ represent the
cost of assigning label $i$ to node $u$, and the \emph{pairwise}
potentials $\theta_{uv}(i,j)$ represent the cost of simultaneously
assigning label $i$ to node $u$ and label $j$ to node $v$.

The MAP inference problem has been extensively studied for special
cases of the potential functions $\theta$. When the pairwise potential
functions $\theta_{uv}$ take the form
\begin{equation*}
  \theta_{uv}(i,j) = \begin{cases}
    0 & i = j\\
    w(u,v) & \text{otherwise,}
  \end{cases}
\end{equation*}
the model is called a \emph{generalized Potts model}. When the
\emph{weights} $w(u,v)$ are nonnegative, as they are throughout this
paper, the model is called \emph{ferromagnetic} or
\emph{attractive}. This formulation has enjoyed a great deal of use in
the computer vision community, where it has proven especially useful
for modeling low-level problems like stereo vision, segmentation, and
denoising \citep{BoykovExpansion, tappen2003comparison}. With this
special form of $\theta_{uv}$, we can re-write the MAP inference
objective as:
\begin{equation}
  \label{eq:umlobj}
  \min_{f: V \to L} Q(f) := \sum_{u \in V}\theta_u(f(u)) + \smashoperator{\sum_{\substack{uv \in E\\f(u)\ne f(v)}}}w(u,v)
\end{equation}
Here we have defined $Q$ as the objective of a feasible labeling
$f$. We can then call $(G, \theta, w, L)$ an instance of MAP inference
for a Potts model with node costs $\theta$ and weights $w$.

The minimization problem \eqref{eq:umlobj} is also known as \uml, and
was first defined and studied under that name by \citet{umlkt}. Exact
minimization of the objective \eqref{eq:umlobj} is NP-hard
\citep{umlkt}, but many good approximation algorithms exist. Most
notably for our work, \citet{umlkt} give a 2-approximation based on
the pairwise LP relaxation \eqref{eq:earthmoverlp}.\footnote{\citet{umlkt} use
  the so-called ``metric LP,'' but this is equivalent to
  \eqref{eq:earthmoverlp} for Potts potentials \citep{EarthmoverLogn,
    LanSonVij18}, and their rounding algorithm also works for this
  formulation.}
\begin{equation}
  \label{eq:earthmoverlp}
  \def\arraystretch{1.2}
  \begin{array}{cll}
    \displaystyle\underset{x}{\text{min}} & \multicolumn{2}{l}{\displaystyle\sum_{u\in V}\displaystyle\sum_{i \in L}\theta_u(i)x_u(i) + \displaystyle\sum_{uv\in E}\sum_{i,j}\theta_{uv}(i,j)x_{uv}(i,j)}\\
    \text{s.t.}\;
    & \displaystyle\sum_i x_{u}(i) = 1, \;
    & \forall u\in V,\ \forall i \in L\\
    & \sum_{j} x_{uv}(i,j) = x_{u}(i)
    & \forall (u,v) \in E,\; \forall i\in L,\\
    & \sum_{i} x_{uv}(i,j) = x_{v}(j)
    & \forall (u,v) \in E,\; \forall j\in L,\\
    & x_u(i) \ge 0,\;
    &\forall u \in V,\ i \in L.\\
    & x_{uv}(i, j) \ge 0,\;
    &\forall (u,v) \in E,\ i,\ j \in L.
  \end{array}
\end{equation}
Their algorithm rounds a solution $x$ of \eqref{eq:earthmoverlp} to a
labeling $f$ that is guaranteed to satisfy $Q(f) \le 2Q(g)$. The
$|V||L|$ decision variables $x_u(i)$ represent the (potentially
fractional) assignment of label $i$ at vertex $u$. While solutions $x$
to \eqref{eq:earthmoverlp} might, in general, take fractional values
$x_u(i) \in (0,1)$, solutions are often found to be \emph{almost
  entirely} binary-valued in practice \citep{Koo, trainandtest,
  swobodapartial, globalMAP, KomoSurvey}, and these values are
typically the same ones taken by the exact solution to the original
problem. Figure \ref{fig:gapfig_appx} demonstrates this phenomenon. In
other words, it is often the case in practice that if $g(u) = i$, then
$x_u(i) = 1$, where $g$ and $x$ are solutions to \eqref{eq:umlobj} and
\eqref{eq:earthmoverlp}, respectively. This property is called
\emph{persistency} \citep{adams1998persistency}. We say a solution $x$
is persistent at $u$ if $g(u) = i$ and $x_u(i) = 1$ for some $i$.

This LP approach to MAP inference has proven popular in practice
because it is frequently persistent on a large percentage of the
vertices in an instance, and because researchers have developed
several fast algorithms for solving \eqref{eq:earthmoverlp}. These
algorithms typically work by solving the \emph{dual}; Tree-reweighted
Message Passing (TRW-S) \citep{kolmogorov2006convergent}, MPLP
\citep{mplp}, and subgradient descent \citep{SonJaaOpt} are three
well-known dual approaches. Additionally, the introduction of fast
general-purpose LP solvers like Gurobi \citep{gurobi} has made it
possible to directly solve the primal \eqref{eq:earthmoverlp} for
medium-sized instances.

\subsection{Stability}
An instance of an optimization problem is stable if its
solution doesn't change when the input is perturbed. To discuss
stability formally, one must specify the exact type of perturbations
considered. As in \citet{LanSonVij18}, we study multiplicative
perturbations of the weights:
\begin{definition}[$(\beta, \gamma)$-perturbation, \citet{LanSonVij18}]
  \label{def:bgpert}
  Given a weight function $w: E \to \mathbb{R}_{\ge 0}$, a weight
  function $w'$ is called a $(\beta, \gamma)$-perturbation $w'$ of
  $w$ if for any $(u,v) \in E$, $$\frac{1}{\beta}w(u,v) \le w'(u,v)
  \le \gamma w(u,v).$$
\end{definition}
With the perturbations defined, we can formally specify the notion of stability:
\begin{definition}[$(\beta, \gamma)$-stable, \citet{LanSonVij18}]
\label{def:strongstab}
  A MAP inference instance $(G, \theta, w, L)$ with graph $G$, node
  costs $\theta$, weights $w$, labels $L$, and integer solution $g$ is
  called $(\beta, \gamma)$-stable if for any $(\beta,
  \gamma)$-perturbation $w'$ of $w$, and any labeling $h \ne g$,
  $Q'(h) > Q'(g),$ where $Q'$ is the objective with costs $c$ and
  weights $w'$.
\end{definition}

That is, $g$ is the unique solution to the optimization
\eqref{eq:umlobj} where $w$ is replaced by any valid $(\beta,
\gamma)$-perturbation of $w$. As $\beta$ and $\gamma$ increase, the
stability condition becomes increasingly strict. One can show that the
LP relaxation \eqref{eq:earthmoverlp} is tight (returns an exact
solution to \eqref{eq:umlobj}) on suitably stable instances:
\begin{theorem}[Theorem 1, \citet{LanSonVij18}]
\label{thm:strongthm}
  Let $x$ be a solution to the LP relaxation \eqref{eq:earthmoverlp} on a
  (2,1)-stable instance with integer solution $g$. Then $x = g$.
\end{theorem}
Many researchers have used stability to understand the real-world
performance of approximation algorithms. \citet{BLstable} introduced
perturbation stability for the \textsc{Max Cut}
problem. \citet{MMWC4Stab} improved their result for \textsc{Max Cut}
and gave a general framework for applying stability to graph
partitioning problems. \citet{LanSonVij18} extended their results to
MAP inference in Potts models.  Stability has also been applied to
clustering problems in machine learning \citep{BBG, WhiteStability,
  BalcanLiangStability, ABSStability, additiveStability}.

\section{BLOCK STABILITY}
\label{sec:blocktight}
The current stability definition used in results for the LP relaxation
(Definition \ref{def:strongstab}) requires that the MAP solution does
not change at all for any $(2,1)$-perturbation of the weights
$w$. This strong condition is rarely satisfied by practical instances
such as those in Figure \ref{fig:gapfig} and Figure \ref{fig:stabfig}.
However, it may be the case that the instance is $(2,1)$-stable when
restricted to \emph{large blocks} of the vertices. We show in Section
\ref{sec:experiments} that this is indeed the case in practice, but
for now we precisely define what it means to be \emph{block stable},
where some parts of the instance may be stable, but others may not. We
demonstrate how to connect the ideas of dual decomposition and
stability, working up to our main theoretical result in Theorem
\ref{thm:blocktight}. Appendix \ref{appx:proofs} contains proofs of
the statements in this section.

We begin our discussion with an informal version of our main theorem:
\begin{itheorem}[see Theorem \ref{thm:blocktight}]
  Assume an instance $(G,\theta,w,L)$ has a block $S$ that
  is $(2,1)$-stable \emph{and} has some additional, additive stability
  with respect to the node costs $\theta$ for nodes \emph{along the
    boundary} of $S$. Then the LP \eqref{eq:earthmoverlp} is
  persistent on $S$.
\end{itheorem}

To reason about different blocks of an instance (and eventually prove
persistency of the LP on them), we need a way to \emph{decompose} the
instance into subproblems so that we can examine each one more or less
independently. The \emph{dual decomposition} framework
\citep{SonJaaOpt, komodakis2011mrf} provides a formal method for doing so.  The commonly
studied Lagrangian dual of \eqref{eq:earthmoverlp}, which we call the
\emph{pairwise dual}, turns every node into its own subproblem:
\begin{equation}
  \label{eq:pairwisedual}
  \begin{split}
    \max_{\eta} P(\eta) = \max_{\eta} \sum_{u \in V}\min_i (\theta_u(i) + \sum_{v}\eta_{uv}(i))\\
    + \sum_{uv \in E}\min_{i,j}(\theta_{uv}(i,j) - \eta_{uv}(i) - \eta_{vu}(j))
  \end{split}
\end{equation}
This can be derived by introducing Lagrange multipliers $\eta$ on the
two consistency constraints for each edge $(u,v) \in E$ and each $i \in L$:
\begin{equation*}
  \begin{split}
    \sum_i x_{uv}(i,j) = x_v(j)\; \forall j\\
    \sum_j x_{uv}(i,j) = x_u(i)\; \forall i
  \end{split}
\end{equation*}
Many efficient solvers for \eqref{eq:pairwisedual} have been
developed, such as MPLP \citep{mplp}. But the subproblems in
\eqref{eq:pairwisedual} are too small for our purposes. We want to
find \emph{large} portions of an instance with stable structure.
Given a set $S \subset V$, define $E_S = \{(u,v) \in E : u\in S,\ v
\in S\}$ to be the set of edges with both endpoints in $S$, and let $T
= V \setminus S$. We may consider relaxing fewer consistency
constraints than \eqref{eq:pairwisedual} does, to form a \emph{block
  dual} with blocks $S$ and $T$.
\begin{equation}
  \begin{split}
  \label{eq:blockdual}
  &\max_{\delta} \sum_{W \in \{S,T\}}\min_{x^W}\left(\sum_{u \in W}\sum_{i \in L}(\theta_u(i) + \smashoperator{\sum_{v:(u,v)\in E_{\partial}}}\delta_{uv}(i))x_u^W(i)\right.\\
  &\left.+\sum_{uv \in E_{W}}\sum_{i,j}\theta_{uv}(i,j)x^W_{uv}(i,j)\right)\\
  &+ \sum_{uv \in E_{\partial}}\min_{i,j}(\theta_{uv}(i,j) - \delta_{uv}(i) - \delta_{vu}(j))
  \end{split}
  \raisetag{2\normalbaselineskip}
\end{equation}
subject to the following constraints for $W \in \{S, T\}$:
\begin{equation}
  \label{eq:blockconstr}
  \def\arraystretch{1.2}
  \begin{array}{ll}
    \displaystyle\sum_i x^W_{u}(i) = 1, \;
    & \forall u\in W,\; \forall i \in L\\
    \sum_{j} x^W_{uv}(i,j) = x^W_{u}(i)
    & \forall (u,v) \in E_{W},\; \forall i\in L,\\
    \sum_{i} x^W_{uv}(i,j) = x^W_{v}(j)
    & \forall (u,v) \in E_{W},\; \forall j\in L.\\
    x^W_u(i) \ge 0,\;
    & \forall u\in W\; \forall i \in L.\\
    x^W_{uv}(i,j) \ge 0,\;
    & \forall (u,v)\in E_W,\; \forall i,\ j \in L.\\    
  \end{array}
\end{equation}
Here the consistency constraints of \eqref{eq:earthmoverlp} are only
relaxed for \emph{boundary edges} that go between $S$ and $T$, denoted
by $E_{\partial}$. Each subproblem (each minimization over $x^W$) is
an LP of the same form as \eqref{eq:earthmoverlp}, but is defined only
on the block $W$ (either $S$ or $T$, in this case). If $S = V$, the
block dual is equivalent to the primal LP \eqref{eq:earthmoverlp}. We
denote the constraint set \eqref{eq:blockconstr} by
$\mathcal{L}^W$. In these subproblems, the node costs $\theta_u(i)$
are modified by $\sum_{v:(u,v)\in E_{\partial}}\delta_{uv}(i)$, the
sum of the block dual variables coming from the other block. We can
thus rewrite each subproblem as an LP of the form:
\[\min_{x^W \in \mathcal{L}^W}\smashoperator{\sum_{u \in W}}\sum_{i \in L}\theta^{\delta}_u(i)x_u^W(i)+\sum_{\mathclap{uv \in
  E_{W}}}\ \sum_{ij}\theta_{uv}(i,j)x^W_{uv}(i,j),\]
where
\begin{equation}
  \label{eq:reparam}
  \theta^{\delta}_u(i) = \theta_u(i) + \sum_{v}\delta_{uv}(i).
\end{equation}
By definition, $\theta^{\delta}$ is equal to $\theta$ on the interior
of each block. It only differs from $\theta$ on the \emph{boundaries}
of the blocks. We show in Appendix \ref{appx:proofs} how to turn a
solution $\eta^*$ of \eqref{eq:pairwisedual} into a solution
$\delta^*$ of \eqref{eq:blockdual}; this block dual is efficiently
solvable. The form of $\theta^{\delta}$ suggests the following definition for
a \emph{restricted instance}:
\begin{definition}[Restricted Instance]
  Consider an instance $(G, \theta, w, L)$ of MAP inference, and let
  $S \subset V$. The instance \emph{restricted to $S$} with
  modification $\delta$ is given by:
  $$((S,E_S), \theta^{\delta}|_S, w|_{E_S}, L),$$ where $\theta^{\delta}$ is as in
  \eqref{eq:reparam} and is restricted to $S$, and the weights $w$ are
  restricted to $E_S$.
\end{definition}

Given a set $S$, let $\delta^*$ be a solution to the block dual
\eqref{eq:blockdual}. We essentially prove that if the instance
restricted to $S$, with modification $\delta^*$, is $(2,1)$-stable,
the LP solution $x$ to the original LP \eqref{eq:earthmoverlp}
(defined on the full, unmodified instance) takes binary values on $S$:
\begin{lemma}
  \label{lem:blocktight}
  Consider the instance $(G, \theta, w, L)$. Let $S \subset V$ be any
  subset of vertices, and let $\delta^*$ be any solution to the block dual
  \eqref{eq:blockdual}. Let $x$ be the solution to
  \eqref{eq:earthmoverlp} on this instance. If the restricted
  instance $$((S,E_S), \theta^{\delta^*}\vert_{S}, w\vert_{E_S}, L)$$
  is $(2,1)$-stable with solution $g_S$, then $x\vert_S = g_S$.
\end{lemma}
Here $g_S$ is the exact solution to the restricted instance $((S,E_S),
\theta^{\delta^*}\vert_{S}, w\vert_{E_S}, L)$ with node costs modified
by $\delta^*$. This may or may not be equal to $g\vert_S$, the overall
exact solution restricted to the set $S$. If $g_S = g\vert_S$, Lemma
\ref{lem:blocktight} implies that the LP solution $x$ is \emph{persistent} on $S$:
\begin{corollary}
  \label{corr:persist}
  For an instance $(G, \theta, w, L)$, let $g$ and $x$ be solutions to
  \eqref{eq:umlobj} and \eqref{eq:earthmoverlp}, respectively. Let
  $S\subset V$ and $\delta^*$ a solution to the block dual for
  $S$. Assume the restricted instance $((S,E_S),
  \theta^{\delta^*}\vert_{S}, w\vert_{E_S}, L)$ is $(2,1)$-stable with
  solution $g\vert_S$. Then $x\vert_S = g\vert_S$; $x$ is persistent
  on $S$.
\end{corollary}
Appendix \ref{appx:proofs} contains a proof of Lemma \ref{lem:blocktight}.

Finally, we can reinterpret this result from the lens of stability by
defining additive perturbations of the node costs $\theta$. Let
$\bar{S}$ be the boundary of set $S$; i.e. the set of $s \in S$ such
that $s$ has a neighbor that is not in $S$.
\begin{definition}[$\epsilon$-bounded cost perturbation]
  \label{def:epspert}
  Given a subset $S \subset V$, node costs $\theta: V \times L \to
  \mathbb{R}$, and a function
  $$\epsilon: \bar{S} \times L \to \mathbb{R},$$ a cost function
  $\theta': V\times L \to \mathbb{R}$ is an $\epsilon$-bounded
  perturbation of $\theta$ (with respect to $S$) if the following
  equation holds for some $\psi$ with $|\psi_u(i)| \le
  |\epsilon_u(i)|$ for all $(u, i) \in V \times L$:
  \begin{equation*}
    \theta'_u(i) = \begin{cases}
      \theta_u(i) + \psi_u(i) & u \in \bar{S}\\
      \theta_u(i) & \text{otherwise.}
    \end{cases}
  \end{equation*}
\end{definition}
In other words, a perturbation $\theta'$ is allowed to differ from
$\theta$ by at most $|\epsilon_u(i)|$ for $u$ in the boundary of $S$,
and must be equal to $\theta$ everywhere else.

\begin{definition}[Stable with cost perturbations]
  \label{def:blockstab}
  A restricted instance $((S,E_S),\theta|_S,w|_{E_S},L)$ with
  solution $g_S$ is called $(\beta,\gamma,\epsilon)$-stable if for all
  $\epsilon$-bounded cost perturbations $\theta'$ of $\theta$, the
  instance $((S,E_S),\theta'|_S,w|_{E_S},L)$ is
  $(\beta,\gamma)$-stable. That is, $g_S$ is the unique solution to
  all the instances $((S,E_S),\theta'|_S,w'|_{E_S},L)$ with $\theta'$ an
  $\epsilon$-bounded perturbation of $\theta$ and $w'$ a
  $(\beta,\gamma)$-perturbation of $w$.
\end{definition}

\begin{theorem}
  \label{thm:blocktight}
  Consider an instance $(G,\theta,w,L)$ with subset $S$, let $g$ and
  $x$ be solutions to \eqref{eq:umlobj} and \eqref{eq:earthmoverlp} on
  this instance, respectively, and let $\delta^*$ be a solution to the
  block dual \eqref{eq:blockdual} with blocks $(S, V\setminus
  S)$. Define $\epsilon^*_u(i) = \sum_{v:(u,v)\in
    E_{\partial}}\delta^*_{uv}(i)$. If the restricted
  instance $$((S,E_S),\theta\vert_S, w\vert_{E_S},L)$$ is
  $(2,1,\epsilon^*)$-stable with solution $g\vert_S$, then the LP $x$ is
  persistent on $S$.
\end{theorem}
\begin{proof}
  This follows immediately from Definition \ref{def:blockstab}, the
  definition of $\epsilon^*$, and Corollary \ref{corr:persist}.
\end{proof}

Definition \ref{def:blockstab} and Theorem \ref{thm:blocktight} provide
the connection between the dual decomposition framework and stability:
by requiring stability to additive perturbations of the node costs
along the boundary of a block $S$, where the size of the perturbation
is determined by the block dual variables, we can effectively isolate
$S$ from the rest of the instance and apply stability to the modified
subproblem.

In Appendix \ref{appx:combine}, we show how to use the dual
decomposition techniques from this section to combine stability with
other structural reasons for persistency of the LP \emph{on the same
  instance}.

\section{FINDING STABLE BLOCKS}
\label{sec:checking}
In this section, we present an algorithm for finding stable blocks in
an instance. We begin with a procedure for testing
$(\beta,\gamma)$-stability as defined in Definition
\ref{def:strongstab}. \citet{LanSonVij18} prove that it is sufficient
to look for labelings that violate stability in the \emph{adversarial
  perturbation}
\begin{equation*}
  w^*(u,v) = \begin{cases}
    \gamma w(u,v) & g(u) \ne g(v)\\
    \frac{1}{\beta} w(u,v) & g(u) = g(v),
  \end{cases}
\end{equation*}
which tries to make the exact solution $g$ as bad as possible. With
that in mind, we can try to find a labeling $f$ such that $f \ne g$,
subject to the constraint that $Q^*(f) \le Q^*(g)$ (here $Q^*$ is the
objective with costs $\theta$ and weights $w^*)$. The instance is
$(\beta, \gamma)$-stable if and only if no such $f$ exists. We can
write such a procedure as the following optimization problem:
\begin{equation}
  \label{eq:mostvio_ILP}
  \def\arraystretch{1.2}
  \begin{array}{cll}
    \displaystyle\underset{x}{\text{max}} & \multicolumn{2}{l}{\frac{1}{2n}\displaystyle\sum_{u\in V}\displaystyle\sum_{i \in L}|x_u(i) - x^g_u(i)|}\\
    \text{s.t.}\;
    &\displaystyle\sum_i x_{u}(i) = 1 \;
    & \forall u\in V,\ \forall i \in L,\\
    &\sum_i x_{uv}(i,j) = x_v(j),\;
    & \forall (u,v) \in E,\ \forall j \in L\\
    &\sum_j x_{uv}(i,j) = x_u(i),\;
    & \forall (u,v) \in E,\ \forall i \in L\\
    &x_u(i) \in \{0,1\}\;
    & \forall u \in V,\ i \in L\\
    &x_{uv}(i,j) \in \{0,1\}\;
    & \forall (u,v) \in E,\ \forall i,j \in L\\
    &Q^*(x) \le Q^*(g)
  \end{array}
\end{equation}
The first five sets of constraints ensure that $x$ forms a feasible
integer labeling $f$.  The objective function captures the normalized
Hamming distance between this labeling $f$ and the solution $g$; it is
linear in the decision variables $x_u$ because $g$ is
fixed---$x^g_u(i) = 1$ if $g(u) = i$ and 0 otherwise.  Of course, the
``objective constraint'' $Q^*(x) \le Q^*(g)$ is also linear in $x$. We
have only linear and integrality constraints on $x$, so we can solve
\eqref{eq:mostvio_ILP} with a generic ILP solver such as Gurobi
\citep{gurobi}. This procedure is summarized in Algorithm
\ref{alg:checkstab}. Put simply, the algorithm tries to find the
labeling $f$ that is most different from $g$ (in Hamming distance)
subject to the constraint that $Q^*(f) \le Q^*(g)$. By construction,
the instance is stable if and only if the optimal objective value of
this ILP is 0. If there is a positive objective value, there is some
$f$ with $f \ne g$ but $Q^*(f) \le Q^*(g)$; this violates
stability. The program is always feasible because $g$ satisfies all
the constraints. Because it solves an ILP, \texttt{CheckStable} is not
a polynomial time algorithm, but we were still able to use it on
real-world instances of moderate size in Section \ref{sec:experiments}.

\begin{algorithm}[t]
  \caption{\texttt{CheckStable}$(g, \beta, \gamma)$}
  \label{alg:checkstab}
  Given $g$, compute the adversarial \bg-perturbation $w^*$ for $g$.

  Construct ILP $\mathcal{I}$ according to \eqref{eq:mostvio_ILP} using $g$ and $w^*$.

  Set $x, d$ = \texttt{GenericILPSolver}$(\mathcal{I})$

  \eIf {$d > 0$} {
    \Return $x$ \emph{// instance is not stable}
  } {
    \Return \texttt{None} \emph{// instances is stable}
  }
\end{algorithm}

We now describe our heuristic algorithm for finding regions of an
input instance that are $(2,1)$-stable after their boundary costs are
perturbed. Corollary \ref{corr:persist} implies that we do \emph{not}
need to test for $(2,1)$-stability for \emph{all} $\epsilon^*$-bounded
perturbations of node costs---we can simply check with respect to the
one given by \eqref{eq:reparam} with $\delta = \delta^*$. That is, we
need only check for $(2,1)$-stability in the instance with node costs
$\theta^{\delta^*}$. This is enough to
guarantee persistency.

In each iteration, the algorithm begins with a partition (henceforth
``decomposition'' or ``block decomposition'') of the nodes $V$ into
disjoint sets $(S_1, \ldots, S_B)$. It then finds a block dual
solution for each $S_b$ (see Appendix \ref{appx:alg} for details) and
computes the restricted instances using the optimal block dual
variables to modify the node costs. Next, it uses Algorithm
\ref{alg:checkstab} to check whether these modified instances are
$(2,1)$-stable. Based on the results of \texttt{CheckStable}, we
either update the previous decomposition or verify that a block is
stable, then repeat.

All that remains are the procedures for initializing the algorithm and
updating the decomposition in each iteration given the results of
\texttt{CheckStable}. The initial decomposition consists of $|L| + 1$
blocks, with
\begin{equation}
  \label{eq:blockdecomp}
  S_b = \{u | g(u) = b \text{ and } \forall (u,v) \in E,\  g(v) = b\}.
\end{equation}
So $|L|$ blocks consist of the interiors of the label sets of $g$---a
vertex $u$ belongs to $S_b$ if $u$ and all its neighbors have
$g(\cdot) = b$. The boundary vertices---$u \in V$ such that there is
some $(u,v) \in E$ with $g(u) \ne g(v)$---are added to a special
\emph{boundary block} denoted by $S_*$. Some blocks may be empty if
$g$ is not surjective.

In an iteration of the algorithm, for every block,
\texttt{CheckStable} returns a labeling $f_b$ that satisfies
$Q_{S_b}^{\theta',w^*}(f_b) \le Q_{S_b}^{\theta',w^*}(g\vert_{S_b})$
and might also have $f_b \ne g\vert_{S_b}$. If $f_b = g\vert_{S_b}$,
the block is stable and we do nothing. Otherwise, we \emph{remove} the
vertices $V_{\Delta} = \{u \in S_b : f_b(u) \ne g\vert_{S_b}(u)\}$ and
add them to the boundary block $S_*$.

Finally, we try to \emph{reclaim} vertices from the old boundary
block. Like all the other blocks, the boundary block gets tested for
stability in each iteration. Some of the vertices in this block may
have $f_b(u) = g\vert_{S_b}(u)$. We call this the \emph{remainder set}
$R$. We run breadth-first-search in $R$ to identify connected
components of vertices that get the same label from $g$. Each of these
components becomes its own new block, and is added to the block
decomposition for the next step. This heuristic prevents the boundary
block from growing too large and significantly improves our
experimental results, since the boundary block is rarely
stable. The entire procedure is summarized in Algorithm
\ref{alg:blockstab}.

\begin{algorithm}[t]
  \caption{\texttt{BlockStable}$(g, \beta, \gamma)$}
  \label{alg:blockstab}
  Given $g$, create blocks $(S^1_1, \ldots, S^1_{k}, S^1_*)$ with
  \eqref{eq:blockdecomp}.

  Initialize $K^1 = |L|$.

  \For{$t \in \{1,\ldots, M\}$} {

    Initialize $S^{t+1}_* = \emptyset$.

    \For{$b \in \{1,\ldots, K^t, *\}$} {
      Find block dual solution $\delta^*$ for $(S^t_b, V\setminus S^t_b)$.
      
      Form $\mathcal{I} = ((S^t_b, E_{S_b}), \theta'\vert_{S^t_b},
      w\vert_{E_{S_b}}, L)$ using $\delta^*$ and \eqref{eq:reparam}.

      Set $f_b =$ \texttt{CheckStable}($g\vert_{S_b^t}, \beta, \gamma$) run on instance
      $\mathcal{I}$.

      Compute $V_{\Delta} = \{u \in S^t_b | f_b(u) \ne g(u)\}$.

      Set $S^{t+1}_b = S^t_b \setminus V_{\Delta}$

      Set $S^{t+1}_* = S^{t+1}_* \cup V_{\Delta}$.

      \If{$b = *$} {
        Set $R = S^t_* \setminus V_{\Delta}$.

        Let $(S^{t+1}_{K^t+1}, \ldots, S^{t+1}_{K^t+p+1})$ = \texttt{BFS}$(R)$ be the $p$ connected components in $R$ that get the same label from $g$.

        Set $K^{t+1} = K^t + p$.
      }
    }
  }
\end{algorithm}

\section{EXPERIMENTS}
\label{sec:experiments}
We focus in this section on instances where the pairwise LP performs
very well. The examples studied here are more extensively examined in
\citet{KomoSurvey}, where they also compare the effectiveness of the
LP to other MAP inference algorithms. Most importantly, though, they
observe that the pairwise LP takes fractional values only at a very
small percentage of the nodes on these instances. This makes them good
candidates for a stability analysis.
\begin{figure*}
  \begin{subfigure}{.16\linewidth}
    \centering
    \includegraphics[width=\linewidth, height=60px]{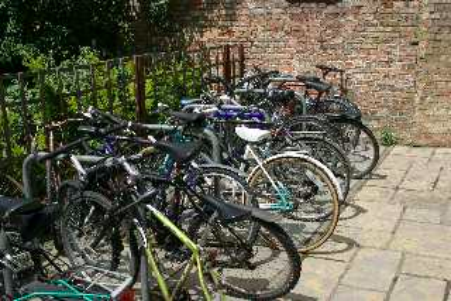}
  \end{subfigure}%
  \begin{subfigure}{.16\linewidth}
    \centering
    \includegraphics[width=\linewidth, height=60px]{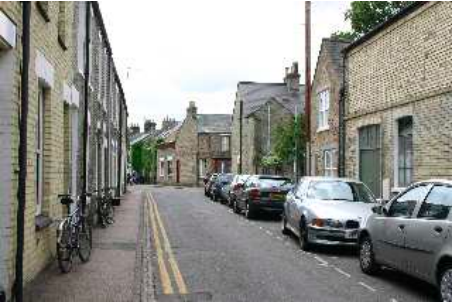}
  \end{subfigure}%
  \begin{subfigure}{.16\linewidth}
    \centering
    \includegraphics[width=\linewidth, height=60px]{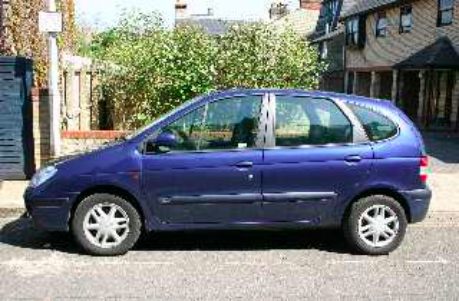}
  \end{subfigure}%  
  \begin{subfigure}{.16\linewidth}
    \centering
    \includegraphics[width=\linewidth, height=60px]{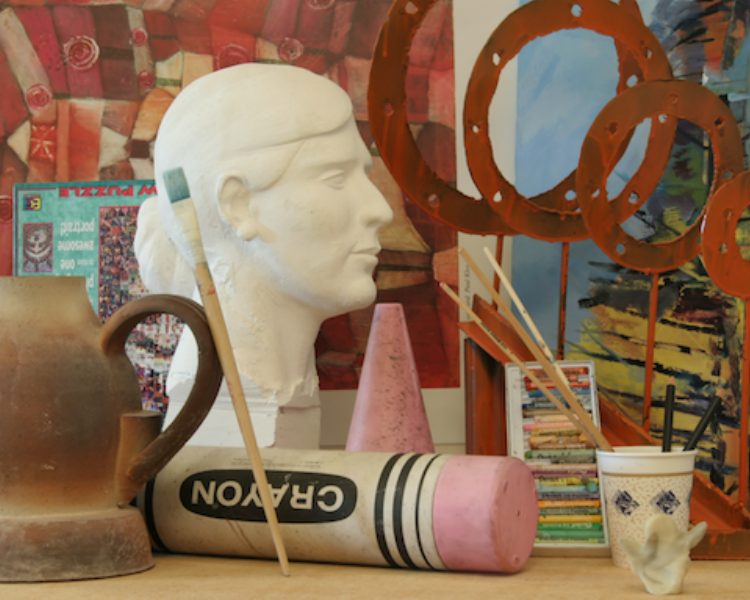}
  \end{subfigure}%
  \begin{subfigure}{.16\linewidth}
    \centering
    \includegraphics[width=\linewidth, height=60px]{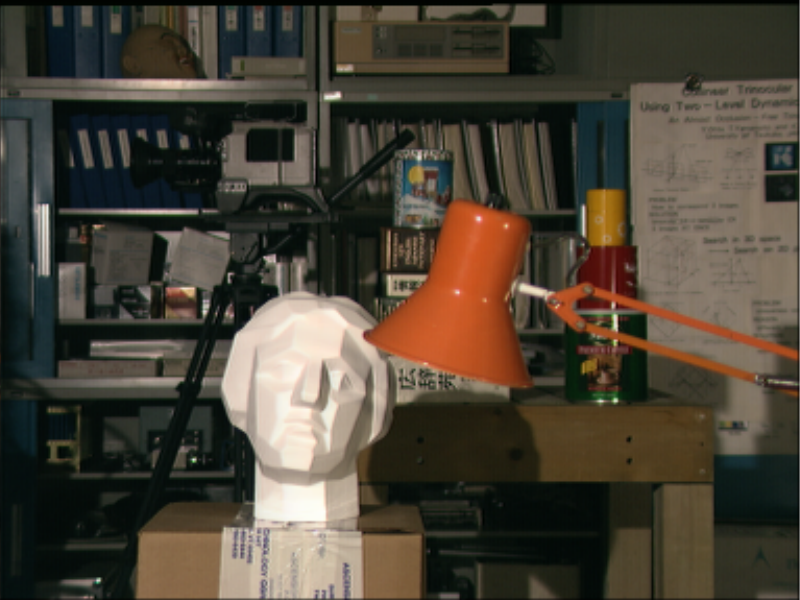}
  \end{subfigure}%
  \begin{subfigure}{.16\linewidth}
    \centering
    \includegraphics[width=\linewidth, height=60px]{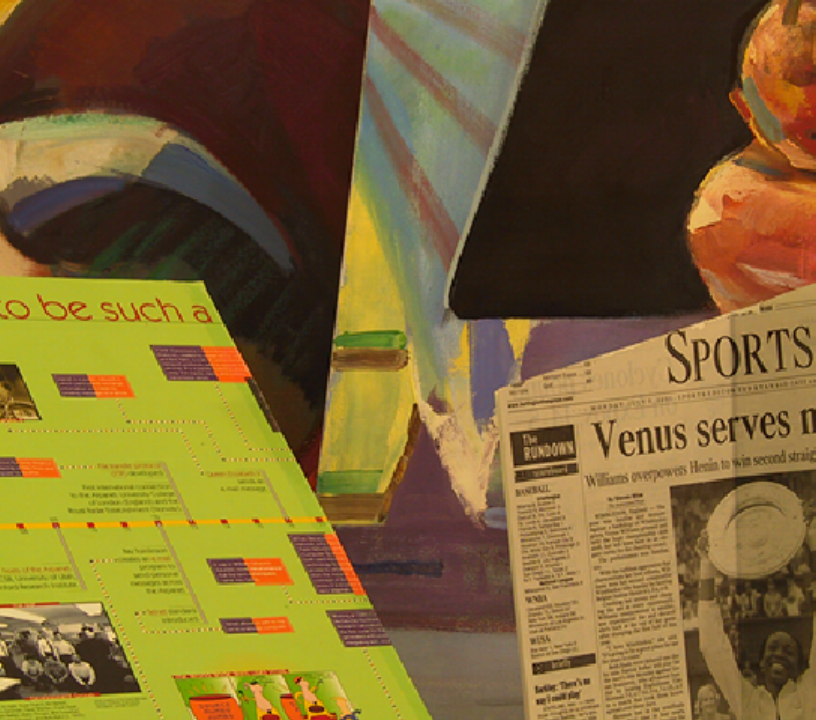}
  \end{subfigure}\\[1ex]
  \begin{subfigure}{.16\linewidth}
    \centering
    \includegraphics[width=\linewidth, height=60px]{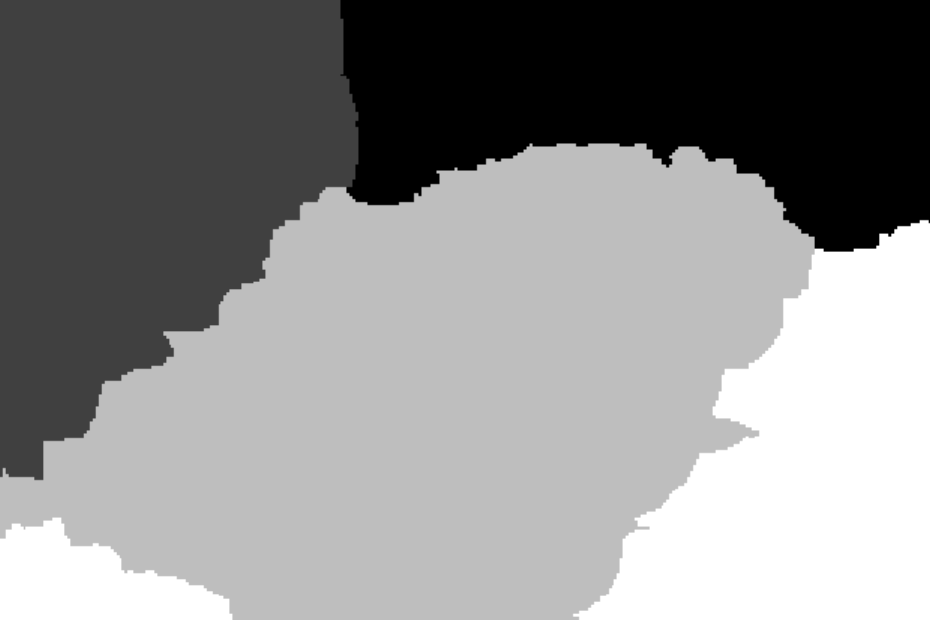}
  \end{subfigure}%
  \begin{subfigure}{.16\linewidth}
    \centering
    \includegraphics[width=\linewidth, height=60px]{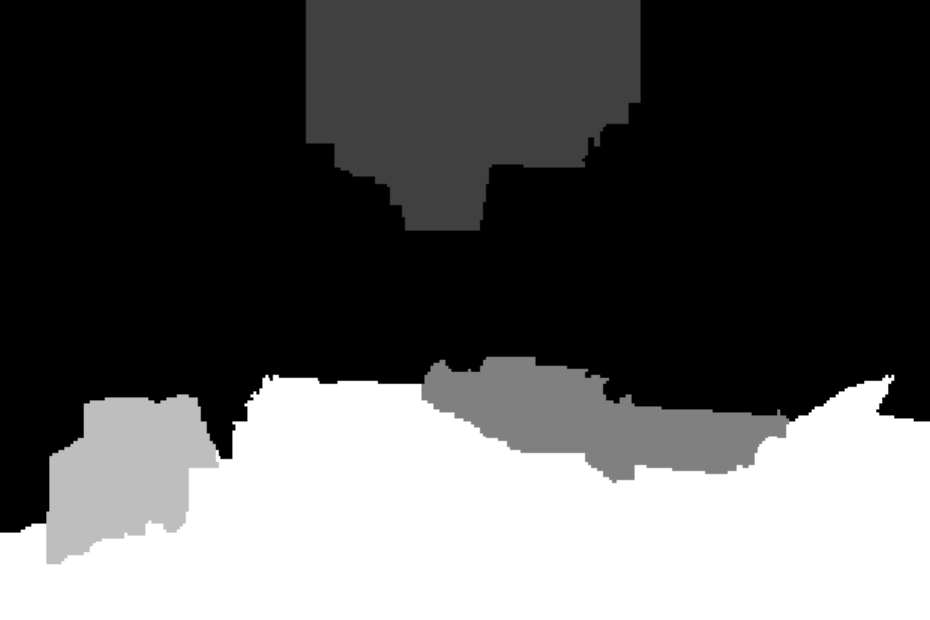}
  \end{subfigure}%
  \begin{subfigure}{.16\linewidth}
    \centering
    \includegraphics[width=\linewidth, height=60px]{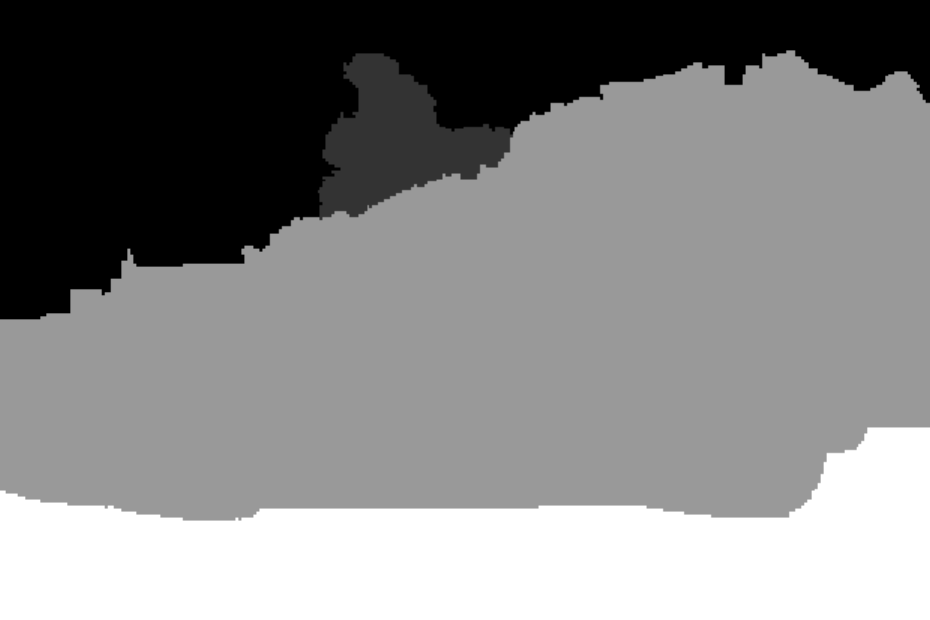}
  \end{subfigure}%  
  \begin{subfigure}{.16\linewidth}
    \includegraphics[width=\linewidth, height=60px]{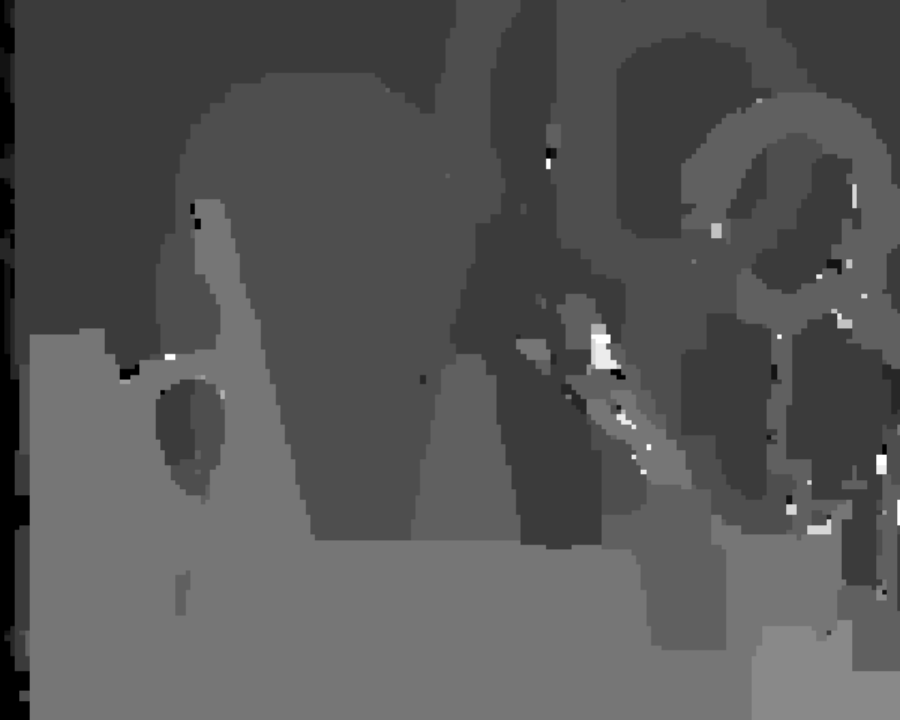}
  \end{subfigure}%
  \begin{subfigure}{.16\linewidth}
    \centering
    \includegraphics[width=\linewidth, height=60px]{144_glabels.pdf}
  \end{subfigure}%
  \begin{subfigure}{.16\linewidth}
    \centering
    \includegraphics[width=\linewidth, height=60px]{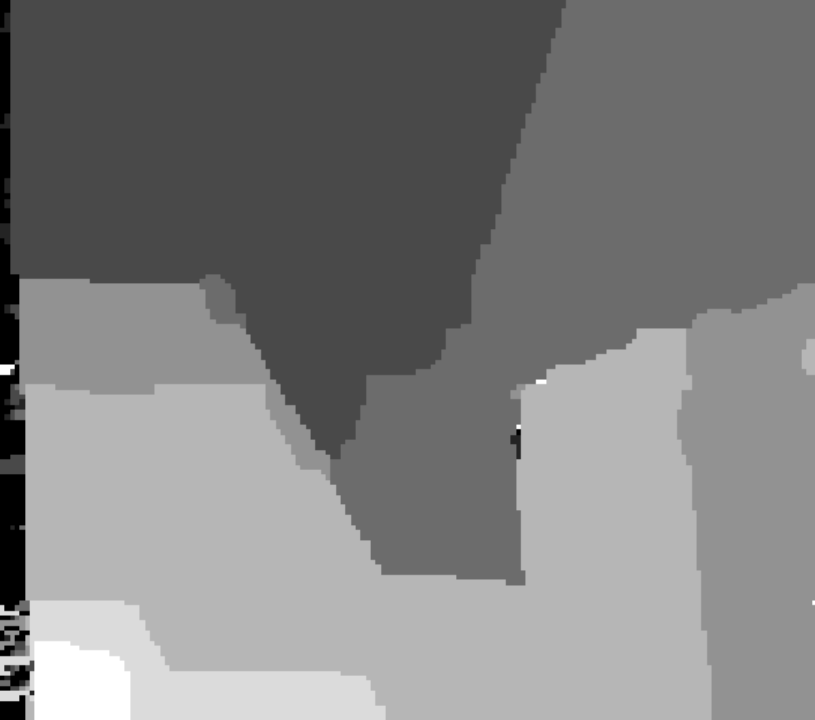}
  \end{subfigure}\\[1ex]
  \begin{subfigure}{.16\linewidth}
    \centering
    \includegraphics[width=\linewidth, height=60px]{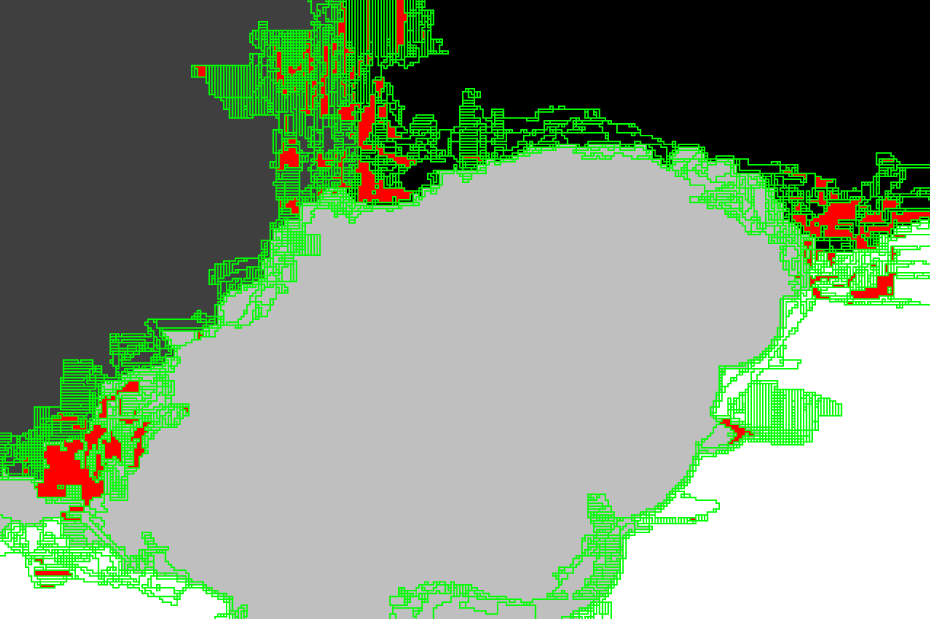}
  \end{subfigure}%
  \begin{subfigure}{.16\linewidth}
    \centering
    \includegraphics[width=\linewidth, height=60px]{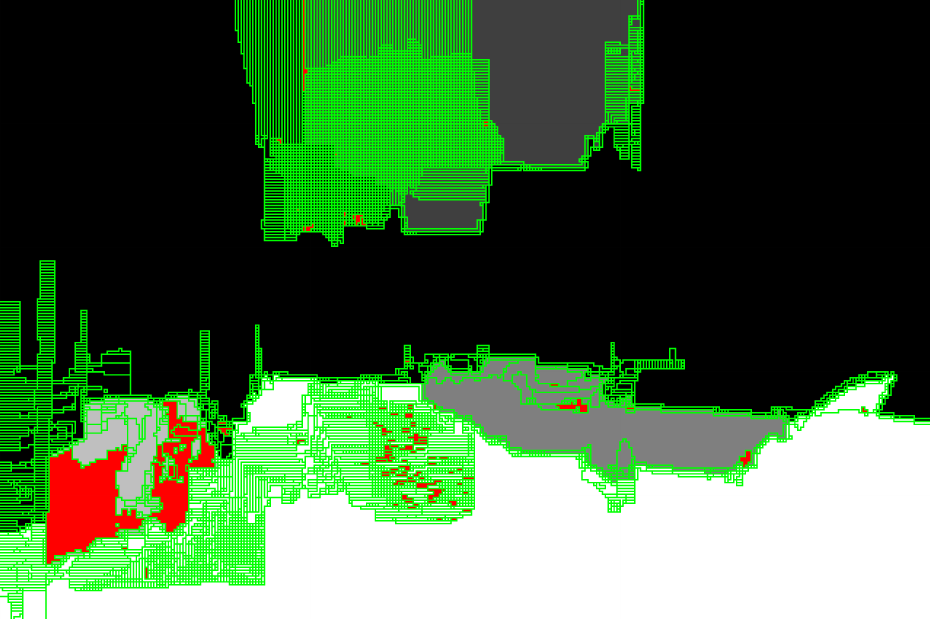}
  \end{subfigure}%
  \begin{subfigure}{.16\linewidth}
    \centering
    \includegraphics[width=\linewidth, height=60px]{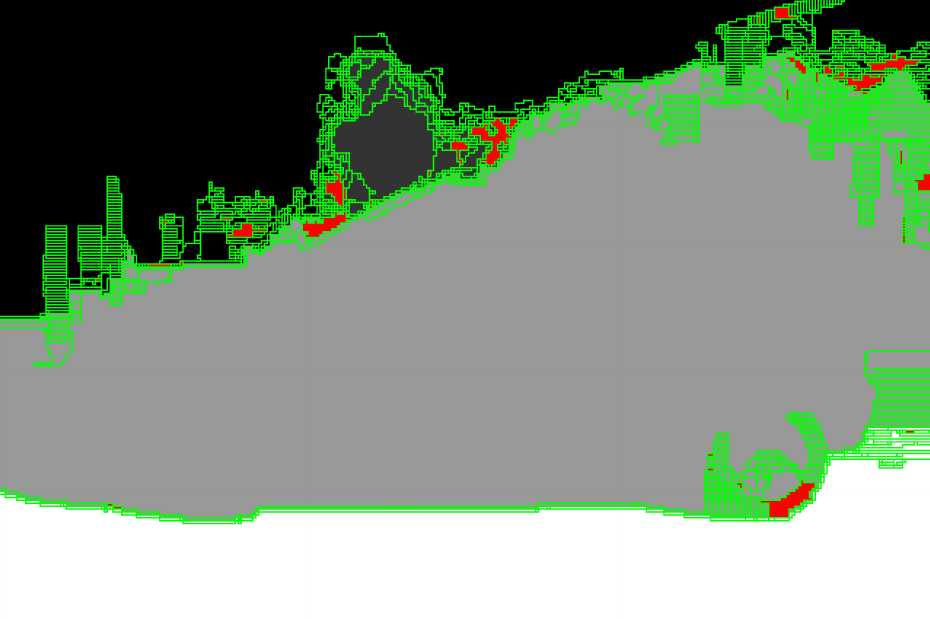}
  \end{subfigure}%  
  \begin{subfigure}{.16\linewidth}
    \includegraphics[width=\linewidth, height=60px]{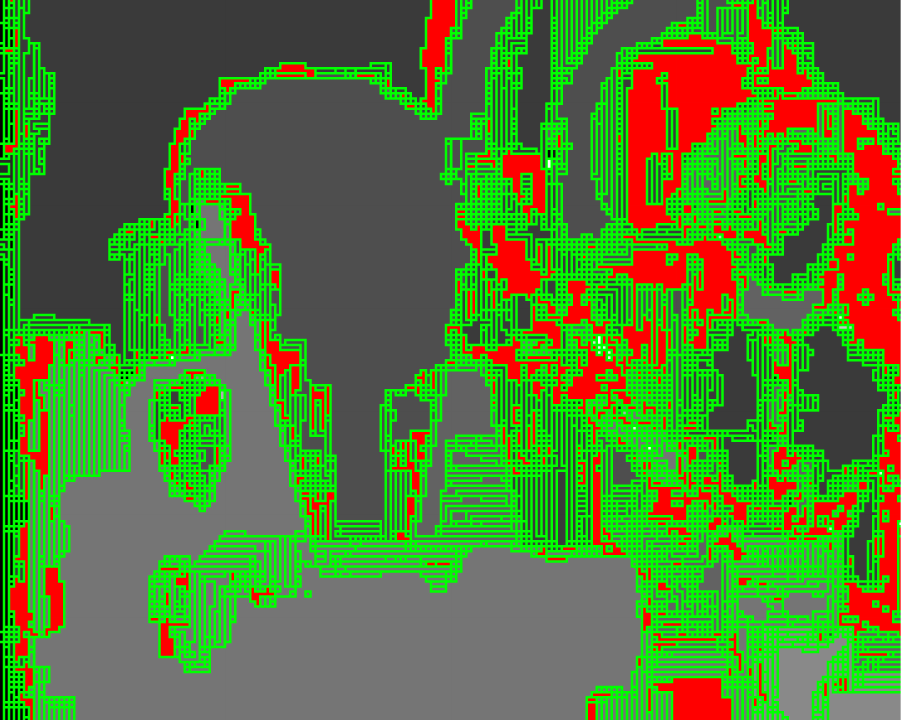}
  \end{subfigure}%
  \begin{subfigure}{.16\linewidth}
    \centering
    \includegraphics[width=\linewidth, height=60px]{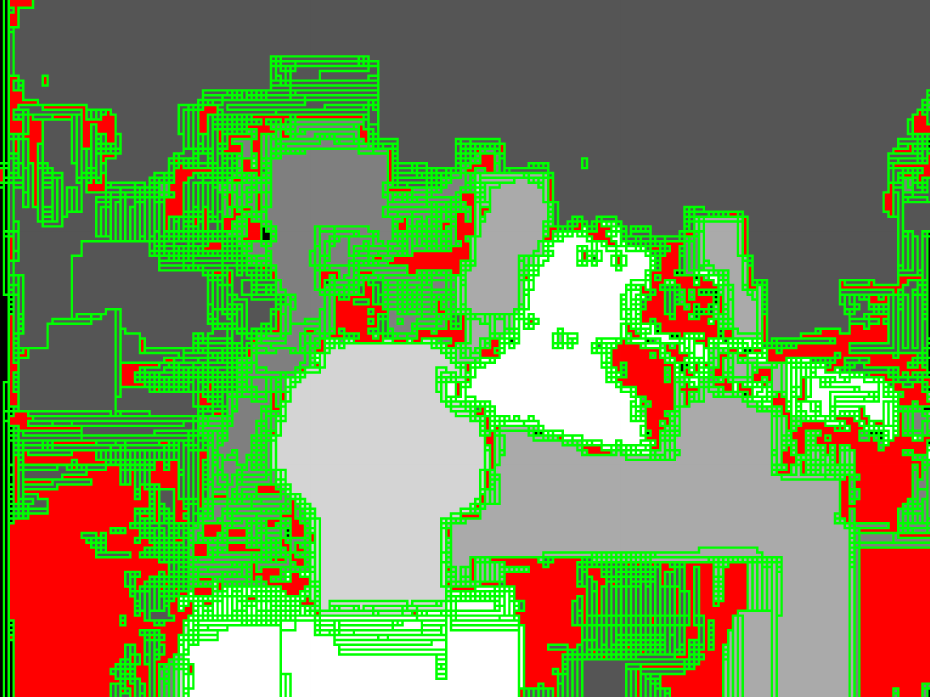}
  \end{subfigure}%
  \begin{subfigure}{.16\linewidth}
    \centering
    \includegraphics[width=\linewidth, height=60px]{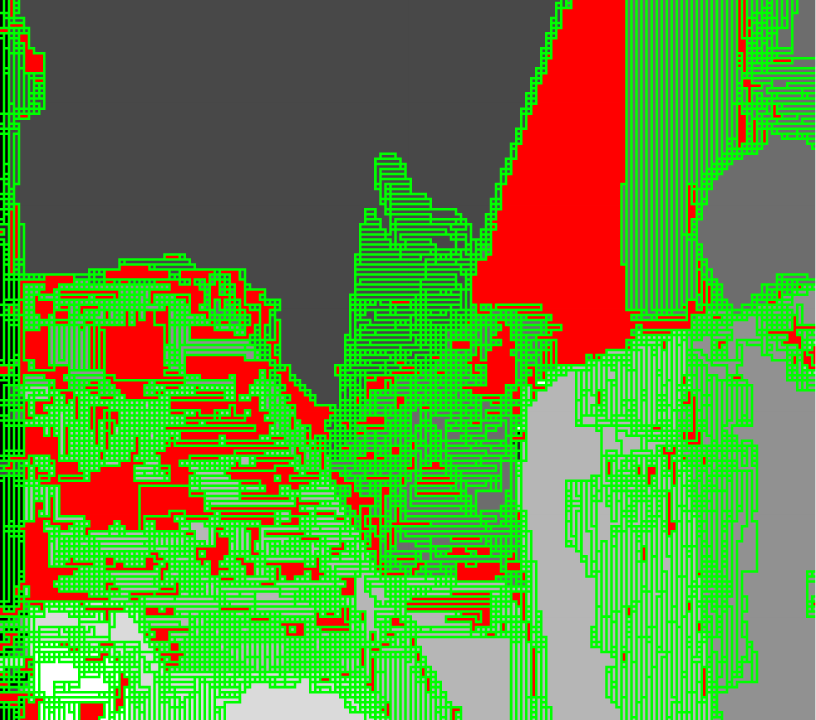}
  \end{subfigure}
  \caption{Columns 1-3: object segmentation instances; Bikes, Road,
    Car. Columns 4-6: stereo instances; Art, Tsukuba, Venus. Row 1:
    original image for the instance. Row 2: MAP solution for the
    model. Row 3: results of Algorithm \ref{alg:blockstab}. Regions
    where the algorithm failed to find a nontrivial stable
    decomposition are shown in red. Boundaries between blocks are
    shown in green. }
  \label{fig:instances}
\end{figure*}

\subsection{Object Segmentation}
For the object segmentation problem, the goal is to partition the
pixels of the input image into a handful of different \emph{objects}
based on the semantic content of the image. The first two rows of
figure \ref{fig:instances} show some example object segmentation
instances.  We study a version of the segmentation problem where the
number of desired objects is known. We use the model of
\citet{alahari2010dynamic}; full details about the MRFs used in this
experiment can be found in Appendix \ref{appx:exp}. Each instance has
68,160 nodes and either five or eight labels, and we ran Algorithm
\ref{alg:blockstab} for $M=50$ iterations to find $(2,1)$-stable
blocks. The LP \eqref{eq:earthmoverlp} is persistent on 100\% of the
nodes for all three instances we study.

Row 3 of Figure \ref{fig:instances} shows the output of Algorithm
\ref{alg:blockstab} on each segmentation instance. The red vertices
are regions where the algorithm was unable to find a large stable
block. The green pixels represent a boundary between blocks,
demonstrating the block structure. The largest blocks seem to
correspond to objects in the original image (and regions in the MAP
solution).

One interesting aspect of these instances is the large number of
stable blocks $S$ with $|S|=1$ for the Road instance (Column 2). If
the LP is persistent at a node $u$, there is a trivial decomposition
in which $u$ belongs to its own stable block (see Appendix
\ref{appx:blocksize} for discussion on block size). However, the
existence of stable blocks with size $|S| > 1$ is not implied by
persistency, so the presence of such blocks means the instances have
special structure. The red regions in Figure \ref{fig:instances}, Row
3 could be replaced by stable blocks of size one. However, Algorithm
\ref{alg:blockstab} did not \emph{find} the trivial decomposition for
those regions, as it did for the center of the Road instance. We
believe the large number of blocks with $|S|=1$ for the Road instance
could therefore be due to our ``reclaiming'' strategy in Algorithm
\ref{alg:blockstab}, which does not try to merge together reclaimed
blocks, rather than a lack of stability in that region.

\subsection{Stereo Vision}
As we discussed in Section \ref{sec:intro}, the stereo vision problem
takes as input two images $L$ and $R$ of the same scene, where $R$ is
taken from slightly to the right of $L$. The goal is to output a depth
label for each pixel in $L$ that represents how far that pixel is from
the camera. Depth is inversely proportional to the disparity (how much
the pixel moves) of the pixel between the images $L$ and $R$. So the
goal is to estimate the (discretized) disparity of each pixel. The
first two rows of Figure \ref{fig:instances} show three example
instances and their MAP solutions. We use the MRF formulation of
\citet{BoykovExpansion} and \citet{tappen2003comparison}. The exact
details of these stereo MRFs can be found in Appendix
\ref{appx:exp}. These instances have between 23,472 and 27,684 nodes,
and between 8 and 16 labels. The LP \eqref{eq:earthmoverlp} is
persistent on 98-99\% of each instance.

Row 3 of Figure \ref{fig:instances} shows the results of Algorithm
\ref{alg:blockstab} for the stereo instances. As with object
segmentation, we observe that the largest stable blocks tend to
coincide with the actual objects in the original image. Compared to
segmentation, fewer vertices in these instances seem to belong to
large stable blocks. We believe that decreased resolution may play a
role in this difference. The computational challenge of scaling
Algorithms \ref{alg:checkstab} and \ref{alg:blockstab} to the stereo
model forced us to use downsampled (0.5x or smaller) images to form
the stereo MRFs. Brief experiments with higher resolution suggest that
improving the scalability of Algorithm \ref{alg:blockstab} is an
interesting avenue for improving these results.

The results in Figure \ref{fig:instances} demonstrate that large stable
regions exist in practical instances. Theorem \ref{thm:blocktight}
guarantees that solutions to \eqref{eq:earthmoverlp} are persistent on
these blocks, so stability provides a novel explanation for the
persistency of the LP relaxation in practice.

\section{DISCUSSION}
The block stability framework we presented helps to understand the
tightness and persistency of the pairwise LP relaxation for MAP
inference. Our experiments demonstrate that large blocks of common
computer vision instances are stable. While our experimental results
are for the Potts model, our extension from $(\beta,
\gamma)$-stability to block stability uses no special properties of
the Potts model and is completely general. If a $(\beta,
\gamma)$-stability result similar to Theorem \ref{thm:strongthm} is
given for other pairwise potentials, the techniques used here
immediately give the analogous version of Theorem
\ref{thm:blocktight}. Our results thus give a connection between dual
decomposition and stability.

The method used to prove the results in Section \ref{sec:blocktight}
can even extend beyond stability. We only need stability to apply
Theorem \ref{thm:strongthm} to a modified block. Instead of stability,
we could plug in any result that guarantees the pairwise LP on that
block has a \emph{unique} integer solution. Appendix \ref{appx:theory}
gives an example of incorporating stability with tree structure on the
same instance. Combining different structures to fully explain
persistency on \emph{real-world} instances will require new
algorithmic insight.

The stability of these instances suggests that designing new inference
algorithms that directly take advantage of stable structure is an
exciting direction for future research. The models examined in our
experiments use mostly hand-set potentials. In settings where the
potentials are learned from training data, is it possible to encourage
stability of the learned models?
\fxnote{summarize contributions again?}

\subsection*{Acknowledgments}
The authors would like to thank Fredrik D. Johansson for his insight
during many helpful discussions. This work was supported by NSF AitF
awards CCF-1637585 and CCF-1723344. AV is also supported by NSF Grant
No.~CCF-1652491.

\subsection*{References}
\begingroup
\renewcommand{\section}[2]{}%
\bibliography{uml}
\endgroup
\newpage
\newpage
\clearpage
\appendix
\section{Theory Details}
\label{appx:theory}

In this appendix, we give a complete exposition and proof of Lemma
\ref{lem:blocktight} and use it to prove Theorem \ref{thm:blocktight}
from Section \ref{sec:blocktight}. We also discuss a subtlety regarding
the size of stable blocks, and show that adding perturbations to the
node costs seems necessary to prove Lemma \ref{lem:blocktight}.

\subsection{Proofs of Lemma \ref{lem:blocktight} and Theorem \ref{thm:blocktight}}
\label{appx:proofs}
We now more formally develop the connection between the block dual
\eqref{eq:blockdual} and block stability. To begin, the \emph{pairwise
  dual} of the LP \eqref{eq:earthmoverlp} is given by:
\begin{equation}
  %\label{eq:pairwisedual}
  \begin{split}
    \max_{\eta} P(\eta) = \max_{\eta} \sum_{u \in V}\min_i (\theta_u(i) + \sum_{v}\eta_{uv}(i))\\
    + \sum_{uv \in E}\min_{i,j}(\theta_{uv}(i,j) - \eta_{uv}(i) - \eta_{vu}(j))
  \end{split}
\end{equation}
This can be derived by introducing Lagrange multipliers $\eta$ on the
two consistency constraints for each edge $(u,v) \in E$ and each $i \in L$:
\begin{equation*}
  \begin{split}
    \sum_i x_{uv}(i,j) = x_v(j)\; \forall j\\
    \sum_j x_{uv}(i,j) = x_u(i)\; \forall i
  \end{split}
\end{equation*}
 A dual point $\eta$ is said to be \emph{locally decodable} at a
 node $u$ if the cost terms $$\theta_u(i) + \sum_{v:uv \in
   E}\eta_{uv}(i)$$ have a unique minimizing label $i$.  This dual $P$
 has the following useful properties for studying persistency of the
 LP \eqref{eq:earthmoverlp}:
\begin{property}[Strong Duality]
  \label{prop:strongdual}
  A solution $\eta^*$ to the maximization \eqref{eq:pairwisedual} has
  $P(\eta^*) = Q(x)$, where $x$ is a solution to the pairwise LP
  \eqref{eq:earthmoverlp}. Here $Q(x)$ is the objective function of
  \eqref{eq:earthmoverlp}; this is identical to $Q$ from
  \eqref{eq:umlobj} when $x$ is integral.
\end{property}
\begin{property}[Complementary Slackness, \citet{SonJaaOpt} Theorem 1.2]
  \label{prop:compslack}
  If $x$ is a primal solution to the pairwise LP \eqref{eq:earthmoverlp}
  and there exists a dual solution $\eta^*$ that is locally decodable at
  node $u$ to label $i$, then $x_u(i) = 1$. That is, if the dual
  solution $\eta^*$ is locally decodable at node $u$, the primal
  solution $x$ is not fractional at node $u$.
\end{property}

\begin{property}[Strict Complementary Slackness, \citet{SonJaaOpt} Theorem 1.3]
  \label{prop:strictcompslack}
  If the LP \eqref{eq:earthmoverlp} has a unique, integral solution
  $x$, there exists a dual solution $\eta^*$ to \eqref{eq:pairwisedual}
  that is locally decodable to $x$.
\end{property}
In particular, Property \ref{prop:compslack} says that to prove the
primal LP is persistent at a vertex $u$, we need only exhibit a dual
solution $\eta^*$ to \eqref{eq:pairwisedual} that is locally decodable
at $u$ to $g(u)$, where $g$ is an integer MAP solution. Properties
\ref{prop:strongdual} and \ref{prop:strictcompslack} will be useful
for proving results about a different Lagrangian dual that relaxes
fewer constraints, which we study now.

Given a partition $V = (S_1, \ldots S_B)$ (henceforth a ``block
decomposition''), we may consider relaxing fewer consistency
constraints than \eqref{eq:pairwisedual} does, to form a \emph{block
  dual}.
\begin{equation}
  \begin{split}
    \label{eq:blockdual_gen}
    &\max_{\delta} B(\delta) :=\\
  &\max_{\delta}\sum_{b} \min_{x^b}\left(\sum_{u \in S_b}\sum_{i \in L}\left(\theta_u(i) + \smashoperator{\sum_{v:(u,v)\in E_{\partial}}}\delta_{uv}(i)\right)x_u^b(i)\right.\\
  &\left.+\sum_{uv \in E_{S_b}}\sum_{i,j}\theta_{uv}(i,j)x^b_{uv}(i,j)\right)\\
  &+ \sum_{uv \in E_{\partial}}\min_{i,j}(\theta_{uv}(i,j) - \delta_{uv}(i) - \delta_{vu}(j))
  \end{split}
\end{equation}
subject to the following constraints for all $b \in \{1,\ldots,B\}$:
\begin{equation}
  \label{eq:blockconstr_gen}
  \def\arraystretch{1.5}
  \begin{array}{ll}
    \displaystyle\sum_i x^b_{u}(i) = 1, \;
    & \forall u\in S_b,\; \forall i \in L\\
    x^b_u(i) \ge 0,\;
    & \forall u\in S_b\; \forall i \in L.\\
    \sum_{j} x^b_{uv}(i,j) = x^b_{u}(i)
    & \forall (u,v) \in E_{S_b},\; \forall i\in L,\\
    \sum_{i} x^b_{uv}(i,j) = x^b_{v}(j)
    & \forall (u,v) \in E_{S_b},\; \forall j\in L.\\
  \end{array}
\end{equation}
This is simply a more general version of the dual
\eqref{eq:blockdual}, written for an arbitrary partition $V = (S_1,
\ldots, S_B)$. Here the consistency constraints are only relaxed for
edges in $E_{\partial}$ (boundary edges, which go from one block to
another). The dual subproblems in the first term of
\eqref{eq:blockdual_gen} are LPs on each block, where the node costs
of boundary vertices are modified by the block dual variables
$\delta$. For any $\delta$, we can define the \emph{reparametrized}
costs $\theta^{\delta}_u$ as
\begin{equation*}
  \theta^{\delta}_u(i) = \begin{cases}
    \theta_u(i) + \sum_{v : (u,v) \in E_{\partial}}\delta_{uv}(i) & \exists (u,v) \in E_{\partial}\\
    \theta_u(i) & \text{otherwise}
  \end{cases},
\end{equation*}
so the block dual objective can also be written as
\begin{equation*}
  \begin{split}
    B(\delta) = \sum_{b} \min_{x^b}\left(\sum_{u \in S_b}
    \sum_{i \in L}\theta^{\delta}_u(i)x_u^b(i) +\right.\\
    \left.\sum_{uv \in E_{S_b}}\sum_{i,j}\theta_{uv}(i,j)x^b_{uv}(i,j)\right) \\
    + \sum_{uv \in E_{\partial}}\min_{i,j}(\theta_{uv}(i,j) - \delta_{uv}(i) - \delta_{vu}(j))
  \end{split}
\end{equation*}
When there is only one block, equal to $V$, the block dual is
equivalent to the primal LP \eqref{eq:earthmoverlp}. When every vertex
is in its own block, the block dual is equivalent to the pairwise dual
\eqref{eq:pairwisedual}.

The following propositions allow us to convert between solutions of
the pairwise dual \eqref{eq:pairwisedual} and the generalized block dual
\eqref{eq:blockdual_gen}.
\begin{proposition}
  \label{lem:pair2block}
  Let $\eta^*$ be a solution to \eqref{eq:pairwisedual}. Let
  $\delta^*$ be the restriction of $\eta^*$ to the domain of $B$; that
  is, $\delta^*_{uv}$ is defined only for pairs $uv$, $vu$ such that
  $(u,v) \in E_{\partial}$ or $(v,u) \in E_{\partial}$:
  \begin{equation*}
    \delta^*_{uv}(i) = \eta^*_{uv}(i) \quad (u,v) \in E_{\partial} \text{ or } (v,u) \in E_{\partial}
  \end{equation*}
  Then $\delta^*$ is a a solution to \eqref{eq:blockdual_gen}.
\end{proposition}
This proposition gives a simple method for converting a solution to
pairwise dual $P$ to a solution to the block dual $B$: simply restrict
it to the domain of $B$. As we explain in Appendix \ref{appx:exp},
this allows us to avoid ever solving the block dual directly; we
simply solve the pairwise dual once, and can then easily form a block
dual solution for any set of blocks.
\begin{proof}
  It is clear that $\delta^*$ defined in this way is dual-feasible
  (there are no constraints on the $\delta$'s). We show that
  $B(\delta^*) \ge P(\eta^*)$. Let $x$ be a primal LP
  solution. Because $B(\delta) \le Q(x)$ for any dual-feasible
  $\delta$ (this is easy to verify), and $P(\eta^*) = Q(x)$ (Property
  \ref{prop:strongdual}), this implies $B(\delta^*) =
  Q(x)$. $\delta^*$ must then be a solution for the block dual $B$.
  Note that this proof also implies strong duality for the block dual.

  To see that $B(\delta^*) \ge P(\eta^*)$, one could observe
  intuitively that $B$ is strictly more constrained than $P$ unless
  every vertex is its own block; since the subproblems are all
  minimization problems, the optimal objective of $B$ will be
  higher. More formally, consider two adjacent nodes $a$ and $b$ in
  the pairwise dual $P$. The terms corresponding to $a$ in $b$ in $P$
  can be written as:
  \begin{align*}
    &\min_{x_a}\sum_i\left(\theta_a(i) + \eta_{ab}^*(i) + \sum_{c : N(a) \setminus \{b\}}\eta^*_{ac}(i)\right)x_a(i) \\
    &+ \min_{x_b}\sum_i\left(\theta_b(i) + \eta_{ba}^*(i) + \sum_{c : N(b) \setminus \{a\}}\eta^*_{bc}(i)\right)x_b(i)\\
    &+ \min_{x_{ab}}\sum_{i,j}\left(\theta_{ab}(i,j) - \eta_{ab}^*(i) -\eta^*_{ba}(j)\right)x_{ab}(i,j),
  \end{align*}
  where $N(u)$ is the set of vertices adjacent to $u$. The $x$ terms
  written here do not appear in \eqref{eq:pairwisedual} because the
  minimum choice at a single vertex $u$ can clearly be chosen by
  $x_u(i) = 1$ for a label $i$ that minimizes the reparametrized
  potential, but we have left them in for convenience (under the
  constraint that $\sum_i x_u(i) = 1$). By the convexity of $\min$,
  the value of the objective above is at most
  \begin{align*}
    &\min_{x_a, x_b, x_{ab}}\sum_i\left(\theta_a(i) + \eta_{ab}^*(i) + \sum_{c : N(a) \setminus \{b\}}\eta^*_{ac}(i)\right)x_a(i)\\
    &+ \sum_i\left(\theta_b(i) + \eta_{ba}^*(i) + \sum_{c : N(b) \setminus \{a\}}\eta^*_{bc}(i)\right)x_b(i)\\
    &+ \sum_{i,j}\left(\theta_{ab}(i,j) - \eta_{ab}^*(i) -\eta^*_{ba}(j)\right)x_{ab}(i,j),
  \end{align*}
  Adding a new constraint to this minimization problem can only
  increase the objective value, so the value of the objective above is
  at most the value of:
  \begin{align*}
    &\min_{x_a, x_b, x_{ab}}\sum_i\left(\theta_a(i) + \sum_{c : N(a) \setminus \{b\}}\eta^*_{ac}(i)\right)x_a(i)\\
    &+ \sum_i\left(\theta_b(i) + \sum_{c : N(b) \setminus \{a\}}\eta^*_{bc}(i)\right)x_b(i)\\
    &+ \sum_{i,j}\theta_{ab}(i,j)x_{ab}(i,j)
  \end{align*}
  subject to the constraints $\sum_j x_{ab}(i,j) = x_a(i)$ for all $i$
  and $\sum_i x_{ab}(i,j) = x_b(j)$ for all $j$. Now the vertices $a$
  and $b$ have been combined into a block. One can continue in this
  way, enforcing consistency constraints within blocks, until arriving
  at:
  \begin{align*}
      &\sum_{u \in V}\min_i (\theta_u(i) + \sum_{v}\eta^*_{uv}(i)) + \sum_{uv \in E}\min_{i,j}(\theta_{uv}(i,j)\\
      &\;\;- \eta^*_{uv}(i) - \eta^*_{vu}(j)) \le\\
      &\sum_{b} \min_{x^b}\left(\sum_{u \in S_b} \sum_{i \in L}\left(\theta_u(i) + \sum_{v}\eta^*_{uv}(i)\right)x_u^b(i)\right.\\
      &\left.\;\;+ \sum_{uv \in E_b}\sum_{i,j}\theta_{uv}(i,j)x^b_{uv}(i,j)\right)\\
      &\;\;+ \sum_{uv \in E_{\partial}}\min_{i,j}(\theta_{uv}(i,j) - \eta^*_{uv}(i) - \eta^*_{vu}(j)),
  \end{align*}
    where the minimizations over $x^b$ on the right-hand-side are
    subject to the constraints \eqref{eq:blockconstr_gen}. The left-hand
    side is $P(\eta^*)$. The expression on the right hand side is
    precisely the objective of $B(\delta^*)$, since we defined
    $\delta^*$ as the restriction of $\eta^*$ to edges in
    $E_{\partial}$. This completes the proof.
\end{proof}

\begin{corollary}[Strong duality for block dual]
  \label{cor:blockstrong}
  If $x$ is a primal solution and $\delta^*$ is a solution to the
  block dual, $B(\delta^*) = Q(x)$.
\end{corollary}
So we are able to easily convert between a pairwise dual solution and
a solution to the block dual. This will prove convenient for two
reasons: there are many efficient pairwise dual solvers, so we can
quickly find $\eta^*$. Additionally, we can solve the pairwise dual
once and convert the solution $\eta^*$ into solutions $\delta^*$ to
the block dual for \emph{any} block decomposition without having to
recompute a solution. As we mentioned above, this will allow us to
quickly test different block decompositions.

The following proposition allows us to convert a solution to the block
dual to a pairwise dual solution.
\begin{proposition}
  \label{lem:block2pair}
  Let $\delta^*$ be a solution to the block dual \eqref{eq:blockdual_gen}.
  Recall that each subproblem of the block dual is an LP of the same
  form as \eqref{eq:earthmoverlp}. So we can consider the pairwise
  dual $P$ \emph{defined on this subproblem}. For block $b$, let
  $\eta^b$ be a solution to the pairwise dual defined on that block's
  (reparametrized) subproblem. That is,
  \begin{equation*}
    \begin{split}
    \eta^b = &\max_{\eta} \sum_{u \in S_b}\min_i \left(\theta_u(i)
    + \smashoperator{\sum_{v: uv \in E_{S_b}}}\eta_{uv}(i)
    + \smashoperator{\sum_{v: uv \in E_{\partial}}}\delta^*_{uv}(i)\right)\\
    &+ \sum_{uv \in E_{S_b}}\min_{i,j}(\theta_{uv}(i,j)
    - \eta_{uv}(i) - \eta_{vu}(j))
    \end{split}
  \end{equation*}
  Then the point $\eta^*$ defined as
  \begin{equation*}
    \eta^*_{uv}(i) = \begin{cases}
      \eta^b_{uv}(i) & (u,v) \in E_{S_b} \text{ or } (v,u) \in E_{S_b}\\
      \delta^*_{uv}(i) & (u,v) \in E_{\partial} \text{ or } (v,u) \in E_{\partial}
    \end{cases}
  \end{equation*}
  is a solution to \eqref{eq:pairwisedual}.
\end{proposition}
Given a solution $\delta^*$ to the block dual, we use Proposition
\ref{lem:block2pair} to extend it to a solution to the pairwise dual
defined on the full instance; combining $\delta^*$ with pairwise dual
solutions on the subproblems induced by $\delta^*$ and the block
decomposition gives an optimal $\eta^*$.
\begin{proof}
  This is immediate from strong duality of the pairwise dual and the
  block dual (Property \ref{prop:strongdual} and Corollary
  \ref{cor:blockstrong}, respectively).
\end{proof}

With this proposition, we are finally ready to prove Lemma \ref{lem:blocktight}.
\begin{proof}[Proof of Lemma \ref{lem:blocktight}]
    We are given a Potts instance $(G, \theta, w, L)$. Let $\delta^*$
    be a solution to \eqref{eq:blockdual_gen} with $S_1 = S$ and $S_2
    = V \setminus S$. We know the
    sub-instance $$((S,E_S),\theta^{\delta^*}\vert_S,w\vert_{E_S},L)$$
    is $(2,1)$-stable. Let $g_S$ be the exact solution to the instance
    $((S,E_S),\theta^{\delta^*}\vert_S,w\vert_{E_S},L)$. If $g$ is the
    exact solution for $(G,\theta,w,L)$, $g_S$ may or may not be the
    same as $g\vert_S$. For this Lemma, they need not be equal, and we
    just work with $g_S$. Because of the $(2,1)$-stability, Theorem
    \ref{thm:strongthm} implies that $g_S$ is the unique solution to
    the following LP:
\begin{equation*}
  \def\arraystretch{1.2}
  \begin{array}{cll}
    \displaystyle\underset{x^S}{\text{min}} & \multicolumn{2}{l}{\displaystyle\sum_{u\in V}\displaystyle\sum_{i \in L}\theta^{\delta^*}_u(i)x^S_u(i) + \displaystyle\sum_{uv\in E}\sum_{i,j}\theta_{uv}(i,j)x^S_{uv}(i,j)}\\
    \text{s.t.}\;
    & \displaystyle\sum_i x^S_{u}(i) = 1, \;
    & \forall u\in V,\ \forall i \in L\\
    & \sum_{j} x^S_{uv}(i,j) = x^S_{u}(i)
    & \forall (u,v) \in E,\; \forall i\in L,\\
    & \sum_{i} x^S_{uv}(i,j) = x^S_{v}(j)
    & \forall (u,v) \in E,\; \forall j\in L,\\
    & x^S_u(i) \ge 0,\;
    &\forall u \in V,\ i \in L.\\
    & x^S_{uv}(i, j) \ge 0,\;
    &\forall (u,v) \in E,\ i,\ j \in L.
  \end{array}
\end{equation*}
This LP is simply the pairwise LP \eqref{eq:earthmoverlp} defined on
$((S,E_S),\theta^{\delta^*}\vert_S,w\vert_{E_S},L)$. Strict complementary
slackness (Property \ref{prop:strictcompslack}) implies that the
pairwise dual problem defined on
$((S,E_S),\theta^{\delta^*}\vert_S,w\vert_{E_S},L)$ has a solution $\eta^S$ that
is locally decodable to $g_S$. That is, there is some $\eta_S$ with
\begin{equation*}
  \begin{split}
    \eta^S = &\max_{\eta} \sum_{u \in S}\min_i \left(\theta_u(i)
    + \smashoperator{\sum_{v: uv \in E_{S}}}\eta_{uv}(i)
    + \smashoperator{\sum_{v: uv \in E_{\partial}}}\delta^*_{uv}(i)\right)\\
    &+ \sum_{uv \in E_{S}}\min_{i,j}(\theta_{uv}(i,j)
    - \eta_{uv}(i) - \eta_{vu}(j))
  \end{split}
\end{equation*}
and for all $u \in S$, $$\argmin_i \left(\theta_u(i) +
\smashoperator{\sum_{v: uv \in E_{S}}}\eta^S_{uv}(i) +
\smashoperator{\sum_{v: uv \in E_{\partial}}}\delta^*_{uv}(i)\right) =
\{g_S(u)\}.$$ In other words, $g_S(u)$ is the unique minimizer of the
modified node costs at $u\in S$.  By Proposition \ref{lem:pair2block},
we can extend $\eta^S$ and $\delta^*$ to a solution $\eta^*$ to the
pairwise dual \eqref{eq:pairwisedual} defined on $(G, \theta, w,
L)$. This extended solution is locally decodable to $g_S$ on $S$ by
construction. If $x$ is a solution to the primal LP
\eqref{eq:earthmoverlp} defined on $(G,\theta,w,L)$, complementary
slackness (Property \ref{prop:compslack}) implies that $x_u(g_S(u)) =
1$ for all $u \in S$. That is, the LP solution $x$ is equal to $g_S$
on $S$.
\end{proof}

Nothing special was used about the block decomposition $(S,V\setminus
S)$, and indeed Lemma \ref{lem:blocktight} also holds for an arbitrary
decomposition $(S_1, \ldots S_B)$; if the instance restricted to a
block $S_b$ is $(2,1)$-stable after its node costs are perturbed by a
solution $\delta^*$ to the block dual \eqref{eq:blockdual_gen}, the
primal LP is equal on $S_b$ to the exact solution of that restricted
instance.

It is clear from Lemma \ref{lem:blocktight} that if the solutions
$g_S$ to the restricted instances are equal to $g\vert_S$ (the exact
solution to the full problem, restricted to $S$), the primal LP $x$ is
persistent on $S$ (this is formalized in Corollary
\ref{corr:persist}). This is why Theorem \ref{thm:blocktight} requires
that the restricted instance is stable with solution $g\vert_S$.
\begin{proof}[Proof of Theorem \ref{thm:blocktight}]
  Note that a block dual solution $\delta^*$ is a valid
  $\epsilon^*$-bounded perturbation of $\theta$ by the choice of
  $\epsilon^*$ and Definition \ref{def:epspert}. Because we have
  assumed in the statement of the theorem that the solution $g_S$ to
  the restricted instance is equal to the restricted solution
  $g\vert_S$, the result follows directly from Lemma
  \ref{lem:blocktight}.
\end{proof}

\subsection{Do we need dual variables?}
\label{appx:simple}
A simpler definition for block stability would be that a block $S$ is
stable if the instance $$((S, E_S), \theta\vert_S, w\vert_{E_S}, L)$$
is $(2,1)$-stable. Unfortunately, this is not enough to guarantee
persistency. Consider the counterexample in Figure
\ref{fig:blockcounter}.

\begin{figure}[t]
  \centering
  \begin{subfigure}{.5\linewidth}
    \centering
    \tikzstyle{vertex}=[circle, draw=black, very thick, minimum size=5mm]
    \tikzstyle{edge} = [draw=black, line width=1]
    \tikzstyle{weight} = [font=\normalsize]
    \begin{tikzpicture}[scale=2,auto,swap]
      \foreach \pos /\name in {{(0,0)}/u,{(1,0)}/w,{(0.5,0.75)}/v}
      \node[vertex](\name) at \pos{$\name$};
      \foreach \source /\dest /\weight in {u/w/1}
      \path[edge] (\source) -- node[weight] {$\weight$} (\dest);
      \foreach \source /\dest /\weight/\pos in {u/v/1/{above left}, v/w/1/{above right}}
      \path[edge] (\source) -- node[weight, \pos] {$\weight$} (\dest);
    \end{tikzpicture}
  \end{subfigure}%
  \begin{subfigure}{.5\linewidth}
    \centering
    \begin{tabular}{|l|ccc|}
      \hline
      \multicolumn{1}{|c|}{\textbf{Node}} & \multicolumn{3}{|c|}{\textbf{Costs}} \\
      \hline
      u & $\infty$       & 0        & $\varepsilon$ \\
      \hline
      v & 0              & $\infty$ & $\varepsilon$ \\
      \hline
      w & $\varepsilon$  & 0        & $\infty$        \\
      \hline
    \end{tabular}
  \end{subfigure}
  \caption{Instance where each node belongs to a block that is
    $(\infty, \infty)$-stable when the node costs are \emph{not}
    perturbed. The LP solution is fractional everywhere.}
  \label{fig:blockcounter}
\end{figure}
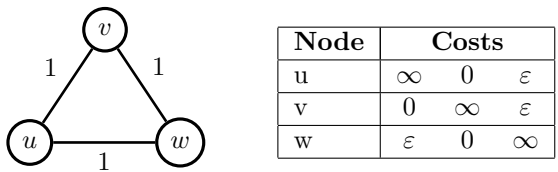
The optimal integer solution $g$ labels $u$ and $w$ with label 2, and
$v$ with label 1, for a total objective of 2. The optimal LP solutions
assigns weight 0.5 to each label with non-infinite cost, for a total
objective of $\frac{3}{2}(1+\varepsilon) < 2$ for any $\varepsilon <
\frac{1}{3}$. Define the block decomposition $S_1 = \{u\}$, $S_2 =
\{v\}$, $S_3 = \{w\}$. Note that each block has a unique optimal
solution given by the minimum-cost label, and that these labels match
the ones assigned in the combined optimal solution $g$. Every vertex
in this instance therefore belongs to an $(\infty, \infty)$-stable
block, according to the simpler definition, but the LP is not
persistent anywhere. It is relatively straightforward to check that
this instance does not satisfy Definition \ref{def:blockstab} or the
conditions of Lemma \ref{lem:blocktight}.

\subsection{Stable block size}
\label{appx:blocksize}
Assume the pairwise dual solution $\eta^*$ is locally decodable on
vertex $u$ to the label $g(u)$, where $g$ is the exact solution. Then
the reparametrized node costs $\theta'_u(i) = \theta_u(i) + \sum_{v\in
  N(u)}\eta^*_{uv}(i)$ have a unique minimizing label $i$. Now
consider solving the block dual \eqref{eq:blockdual_gen} when $S_u =
\{u\}$ is a block with just one vertex, $u$. Around block $S_u$,
$\delta^*_{uv}(i) = \eta^*_{uv}(i)$ is a solution to the block dual
(see Proposition \ref{lem:pair2block}). But this means that $S_u$ is a
$(\infty, \infty)$-stable block with the modified node costs (there
are no edges to perturb, and the node costs have a unique
minimizer). In this way, it is trivial to give a stable block
decomposition any time the LP \eqref{eq:earthmoverlp} is persistent on a
node $u$---simply add $u$ to its own block. However, it is not
possible \emph{a priori} to find stable blocks of size greater than
one, and we show in Section \ref{sec:experiments} that many such
blocks exist in practice. These practical instances therefore exhibit
structure that is more special than persistency: large stable blocks
are not to be expected from persistency alone, and their existence
implies persistency.

\subsection{Combining stability with other structure}
\label{appx:combine}
\begin{figure}[t]
  \centering
  \tikzstyle{vertex}=[circle, draw=black, very thick, minimum size=5mm]
  \tikzstyle{edge} = [draw=black, line width=1]
  \tikzstyle{weight} = [font=\normalsize]
  \begin{tikzpicture}[scale=2,auto,swap]
    \foreach \pos /\name in {{(0,0)}/v,{(1,0)}/w,{(0.5,0.75)}/u, {(2,0.75)}/x, {(2,0)}/y, {(3,0)}/z}
    \node[vertex](\name) at \pos{$\name$};
    \foreach \source /\dest /\weight in {v/w/2, w/y/\epsilon, x/y/2, y/z/2-\gamma}
    \path[edge] (\source) -- node[weight] {$\weight$} (\dest);
    \foreach \source /\dest /\weight/\pos in {u/v/2/{above left}, u/w/2/{above right}, u/x/\epsilon/{above}}
    \path[edge] (\source) -- node[weight, \pos] {$\weight$} (\dest);
  \end{tikzpicture}
  \caption{Potts model instance with both stable and tree structure.}
  \label{fig:combined}
\end{figure}
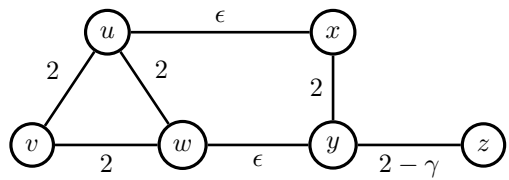
\begin{figure}[t]
  \centering
  \begin{subfigure}[b]{0.5\linewidth}
    \begin{tabular}{|l|ccc|}
      \hline
      \multicolumn{1}{|c|}{\textbf{Node}} & \multicolumn{3}{|c|}{\textbf{Costs}} \\
      \multicolumn{1}{|c|}{} & 1 & 2 & 3\\      
      \hline
      u & 0       & 0        & 2 \\
      \hline
      v & 0              & $\infty$ & $\infty$ \\
      \hline
      w & 0  & 0        & 2        \\
      \hline
      x & 2  & 0        & 2        \\
      \hline
      y & 2  & 0        & 2        \\
      \hline
      z & 0  & 1        & 1        \\
      \hline
    \end{tabular}
    \subcaption{Original node costs $\theta$}
  \end{subfigure}%
  \begin{subfigure}[b]{0.5\linewidth}
    \centering
    \begin{tabular}{|l|c|}
      \hline
      \multicolumn{1}{|c|}{\textbf{Node}} & \multicolumn{1}{|c|}{\textbf{Opt. Label}} \\
      \hline
      u & 1\\
      \hline
      v & 1 \\
      \hline
      w & 1\\
      \hline
      x & 2 \\
      \hline
      y & 2 \\
      \hline
      z & 2 \\
      \hline
    \end{tabular}
    \subcaption{Exact solution $g$}
  \end{subfigure}
  \caption{Details for the instance in Figure \ref{fig:combined}. The
    strictly positive values $\epsilon$ and $\gamma$ are both
    taken sufficiently small.}
  \label{fig:combinedetails}
\end{figure}

Consider the instance in Figure \ref{fig:combined}. The tables in
Figure \ref{fig:combinedetails} give the original node costs $\theta$
and the exact solution $g$ for this instance. The objective of $g$ is
$1 + 2\epsilon$. The pairwise LP \eqref{eq:earthmoverlp} is persistent
on this instance. How can we explain that? The instance is not
$(2,1)$-stable: when the weight between $y$ and $z$ is multiplied by
$\frac{1}{2}$, the optimal label for $z$ switches from 2 to
1. However, if we take $\epsilon$ to be very small, the blocks $S =
\{u,v,w\}$ and $T = \{x,y,z\}$ seem loosely coupled, and the strong
node costs and connections in $S$ suggest it might have some stable
structure. Unfortunately, the block $T$ is not stable for the same
reason that the overall instance is not stable. However, this block is
a tree!

It is fairly straightforward to verify that $\delta^*$ given by
\begin{equation*}
  \begin{aligned}
    \delta^*_{ux} = (\epsilon, 0, 0) &\;\;\;\; \delta^*_{xu} = (-\epsilon, 0, 0)\\
    \delta^*_{wy} = (\epsilon, 0, 0) &\;\;\;\; \delta^*_{yw} = (-\epsilon, 0, 0)
  \end{aligned}
\end{equation*}
is a solution to the block dual with blocks $\{S,T\}$. Indeed, Figure
\ref{fig:combinedetails2} shows the node costs $\theta^{\delta^*}$
updated by this solution. If we solve the LP on each modified block,
ignoring the edges between $S$ and $T$, we get an objective of
$2\epsilon$ for $S$ and an objective of 1 for $T$. Because this
matches the objective of the original exact solution $g$, we know in
this case that $\delta^*$ must be optimal for the block dual.
\begin{figure}[t]
  \centering
  \begin{tabular}{|l|ccc|}
      \hline
      \multicolumn{1}{|c|}{\textbf{Node}} & \multicolumn{3}{|c|}{\textbf{Costs}} \\
      \multicolumn{1}{|c|}{} & 1 & 2 & 3\\
      \hline
      u & $\epsilon$       & 0        & 1 \\
      \hline
      v & 0              & $\infty$ & $\infty$ \\
      \hline
      w & $\epsilon$  & 0        & 1        \\
      \hline
      x & 2$-\epsilon$  & 0        & 2        \\
      \hline
      y & 2$-\epsilon$  & 0        & 2        \\
      \hline
      z & 0  & 1        & 1        \\
      \hline
    \end{tabular}
    \caption{Updated node costs $\theta^{\delta^*}$}
    \label{fig:combinedetails2}
\end{figure}
It can then be shown that the modified block $$((S,E_S),
\theta^{\delta^*}\vert_S, w\vert_{E_S}, L)$$ is $(2,1)$-stable: when
all the weights of edges in $E_S$ become 1 instead of 2, the solution
is still to label $u$ and $w$ with label 1 for sufficiently small
$\phi$ and $\epsilon$.  Similarly, the block $$((T,E_T),
\theta^{\delta^*}\vert_T, w\vert_{E_T}, L)$$ is a tree with a unique
integer solution; because the pairwise LP relaxation is tight on trees
\citep{wainwrightjordan}, this implies by Property
\ref{prop:strictcompslack} that there is a pairwise dual solution to
this restricted instance that is locally decodable. Put together,
these two results explain the persistency of the pairwise LP
relaxation on the \emph{full} instance by applying different structure
at the sub-instance level.

\section{Experimental Details}
\label{appx:exp}

In this appendix, we provide more details and additional discussion
regarding the algorithms and experiments in Sections \ref{sec:checking}
and \ref{sec:experiments}.

\subsection{Explaining Algorithm \ref{alg:blockstab}}
\label{appx:alg}
We briefly give more details on the steps of Algorithm
\ref{alg:blockstab}. One key point is that we can efficiently compute
block dual solutions with very little extra computation per outer
iteration of the algorithm. We effectively only need to solve a dual
problem once; we can then easily generate block dual solutions for any
block decomposition for all subsequent iterations. In practice, we
simply find a pairwise dual solution $\eta^*$ using the MPLP algorithm
\citep{mplp}, then use Proposition \ref{lem:pair2block} to convert it
to a solution of the generalized block dual \eqref{eq:blockdual_gen}
for a given decomposition.

Additionally, we can avoid the expensive component of the inner loop
of the algorithm (solving \texttt{CheckStable} for each block $b$). To
parallelize \texttt{CheckStable} without any additional work, we
modify the node costs of each block using the solution $\delta^*$ to
the generalized block dual, then remove all the edges in
$E_{\partial}$. We can then solve the ILP \eqref{eq:mostvio_ILP} used
in \texttt{CheckStable} with one ``objective constraint'' for each
block. The objective function of \eqref{eq:mostvio_ILP} decomposes
across blocks once $E_{\partial}$ is removed. This approach avoids the
overhead of explicitly forming and solving the ILP
\eqref{eq:mostvio_ILP} for each block, which is especially helpful as
the number of blocks grows large. These optimizations are summarized
in Algorithm \ref{alg:blockstab_eff}.

\begin{algorithm}[t]
  \caption{\texttt{BlockStable}$(g, \beta, \gamma)$ (optimized)}
  \label{alg:blockstab_eff}
  Given $g$, create blocks $(S^1_1, \ldots, S^1_{k}, S^1_*)$ with
  \eqref{eq:blockdecomp}.

  Initialize $K^1 = |L|$.
  
  Find a solution $\eta^*$ to \eqref{eq:pairwisedual}.
  
  \For{$t \in \{1,\ldots, M\}$} {

    Initialize $S^{t+1}_* = \emptyset$.

    Compute $\delta^*$ for $(S^t_1, \ldots S^t_{K^t}, S^t_*)$ using $\eta^*$ and Proposition \ref{lem:pair2block}.

    Form $\mathcal{I} = ((V, E\setminus E_{\partial}), \theta^{\delta^*},
    w\vert_{E\setminus E_{\partial}}, L)$ using $\delta^*$ and
    \eqref{eq:reparam}.
    
    Set $(f_1,\ldots f_{K^t}, f_*) =$ \texttt{CheckStable}$(g, \beta, \gamma)$ run on instance
    $\mathcal{I}$.
    
    \For{$b \in \{1,\ldots, K^t, *\}$} {
      Compute $V_{\Delta} = \{u \in S^t_b | f_b(u) \ne g(u)\}$.

      Set $S^{t+1}_b = S^t_b \setminus V_{\Delta}$

      Set $S^{t+1}_* = S^{t+1}_* \cup V_{\Delta}$.

      \If{$b = *$} {
        Set $R = S^t_* \setminus V_{\Delta}$.

        Let $(S^{t+1}_{K^t+1}, \ldots, S^{t+1}_{K^t+p+1})$ = \texttt{BFS}$(R)$ be the $p$ connected components in $R$ that get the same label from $g$.

        Set $K^{t+1} = K^t + p$.
      }
    }
  }
\end{algorithm}

\subsection{Object Segmentation}

\subsubsection*{Setup: Markov Random Field}
We use the formulation of \citet{shotton2006textonboost,
  alahari2010dynamic}. The graph $G$ is a grid with one vertex for
each pixel in the original image; the edges connect adjacent
pixels. In this model, the node costs $\theta$ are set based on the
location of the pixel in the image, the color values at that pixel,
and the local shape and texture of the image. The edge weights
$w(u,v)$ are set as:
\[
w(u,v) = \eta_1\exp\left(-\frac{||I(u) - I(v)||_2^2}{2\sum_{p,q}||I(p) - I(q)||_2^2}\right) + \eta_2.
\]
Here $\eta = (\eta_1, \eta_2)$, $\eta \ge 0$ are learned parameters, and $I(u)$ is the vector of RGB values for pixel $u$ in the image. \citet{shotton2006textonboost} learn the node and edge parameters using a boosting method.\footnote{We use
  pre-built object segmentation models from the OpenGM Benchmark that
  are based on the models of \citep{alahari2010dynamic}:
  \url{http://hciweb2.iwr.uni-heidelberg.de/opengm/index.php?l0=benchmark}}
This setup yields an instance of a Potts model (\uml), so we can
proceed with our algorithms.  Many vertices of the object segmentation
instances appear to belong to large stable blocks. Unlike with stereo
vision, we were able to use the full instances in our experiments,
which, as we observed in Section \ref{sec:experiments}, could
contribute to the quality of our results for segmentation. Each
instance has 68,160 nodes and either five or eight labels. The LP is
persistent on 100\% of the nodes for all three instances.

\subsection{Stereo Vision}
\subsubsection*{Setup: Markov Random Field}
To begin, we let the graph $G$ be a grid graph where each node
corresponds to a pixel in $L$. We then need to set the costs
$\theta_u(i)$ for each $u$, $i$, and the weights $w(u,v)$ for each
edge $(u,v)$ in the grid. This is where the domain knowledge enters
the problem. For a pixel $u$, we set its cost $\theta_u(i)$ for
disparity $i$ as:
\begin{equation}
  \label{eq:intensdiff}
  \theta_u(i) = \left(I_L(u) - I_R(u - i)\right)^2.
\end{equation}
Here $I_L$ and $I_R$ are the pixel intensity functions for the images
$L$ and $R$, respectively, and the notation $u - i$ shifts a pixel
location $u$ by $i$ pixels to the left. That is, if $u$ corresponds to
location $(h,w)$, $u - i$ corresponds to location $(h, w-i)$. If the
difference \eqref{eq:intensdiff} is high, then it is unlikely that
pixel $u$ actually moved $i$ pixels between the two images. On the
other hand, if this difference is low, disparity $i$ is a plausible
choice for pixel $u$.  In our experiments, we use a small correction
to \eqref{eq:intensdiff} that accounts for image sampling
\citep{birchfield1998pixel}; this correction is also used by
\citet{BoykovExpansion} and \citet{tappen2003comparison}.

We can set the weights using a similar intuition. If $u$ and $v$ are
neighboring pixels and $I_L(u)$ is similar to $I_L(v)$, then $u$ and
$v$ probably belong to the same object, so they should probably get
the same disparity label. In this case, the weight between them should
be high. On the other hand, if $I_L(u)$ is very different from
$I_L(v)$, $u$ and $v$ may not belong to the same object, so they
should have a low weight---they may move different amounts between the
two images. To this end, we set
\begin{equation*}
  w(u,v) = \begin{cases}
    P \times s & |I_L(u) - I_L(v)| < T\\
    s & \text{otherwise}.
  \end{cases}
\end{equation*}
In our experiments, we follow \citet{tappen2003comparison} and set
$s=50$, $P = 2$, $T = 4$. This setup gives us a Potts model instance
$(G, \theta, w, L)$. 

\end{document}